%% file: paper.tex
\PassOptionsToPackage{table,xcdraw}{xcolor}
\documentclass{article}

\usepackage[subtle]{savetrees}

\usepackage{multirow}
\usepackage[table]{xcolor}
\definecolor{lightblue}{rgb}{0.93,0.95,1.0}

\usepackage{mathtools}

\usepackage{colortbl}

\usepackage{mdwmath}
\usepackage{hyperref}

\usepackage[accepted]{icml2024}

\usepackage{array}

\usepackage{environ}
\usepackage{booktabs}

\usepackage{url}

\usepackage{stfloats}

\usepackage{epsfig}
\usepackage{bbm}
\usepackage{tabularx}
\usepackage{rotating}
\usepackage{times}
\usepackage{comment}
\usepackage{amsthm}
\usepackage{pifont}
\usepackage{amsfonts}
\usepackage{footnote}

\usepackage{xspace} 
\usepackage{wrapfig}
\usepackage{setspace}
\usepackage{bm}
\usepackage[export]{adjustbox}
\usepackage{stfloats} 
\usepackage{microtype} 
\usepackage{diagbox}
\usepackage{listings}
\usepackage{thmtools,thm-restate}

\usepackage{arydshln}

\usepackage{caption}
\usepackage[caption=false,font=footnotesize,labelfont=sf,textfont=sf]{subfig}

\usepackage{soul}

\usepackage{enumitem}

\microtypecontext{spacing=nonfrench}

\newcommand{\ie}{i.e., }
\newcommand{\eg}{e.g., }

\newcommand{\code}[1]{\texttt{#1}}

\newtheoremstyle{mydef}
  {.1cm}   
  {0cm}   
  {\normalfont}  
  {0pt}       
  {\bfseries} 
  {.}         
  {5pt plus 1pt minus 1pt} 
  {}          
  
\theoremstyle{mydef}
\newtheorem{definition}{Definition}[section]
\newtheorem{theorem}{Theorem}
\newtheorem{lemma}{Lemma}

\makeatletter
\renewenvironment{proof}[1][\proofname]{\par
  \normalfont \topsep6\p@\@plus6\p@\relax
  \trivlist
  \item[\hskip\labelsep
        \itshape
    #1\@addpunct{.}]\ignorespaces
}{%
  \endtrivlist\@endpefalse
}
\makeatother

\newcommand{\para}[1]{\noindent\textbf{#1}}

\definecolor{Gray}{gray}{0.85}
\definecolor{myblue}{rgb}{0,.6,1}
\definecolor{darkgray}{rgb}{.4,.4,.4}
\definecolor{lightblue}{HTML}{DAE8FC}
\definecolor{lightred-lightest}{HTML}{FDEDEC}
\definecolor{lightred-medium}{HTML}{FADBD8}
\definecolor{lightred-darkest}{HTML}{F5B7B1}
\definecolor{lightgray}{HTML}{B3B3B3}
\definecolor{shadecolor}{HTML}{DCDCDC}
\definecolor{lightorange}{HTML}{FFCC99}
\definecolor{gray95}{gray}{0.05}
\definecolor{code}{HTML}{660066}
\definecolor{ao(english)}{rgb}{0.0, 0.5, 0.0}
\definecolor{codegreen}{rgb}{0,0.6,0}

\newcolumntype{?}[1]{!{\vrule width #1}}

\def\sys{\textsc{SymC}\xspace}
\def\pto{PalmTree\xspace}
\def\pts{PalmTree-O\xspace}
\def\ptu{PalmTree-N\xspace}

\icmltitlerunning{Exploiting Code Symmetries for Learning Program Semantics}

\begin{document}

\twocolumn[
\icmltitle{Exploiting Code Symmetries for Learning Program Semantics}




\begin{icmlauthorlist}
\icmlauthor{Kexin Pei}{col,chi}
\icmlauthor{Weichen Li$^*$}{col}
\icmlauthor{Qirui Jin$^*$}{umich}
\icmlauthor{Shuyang Liu}{uiuc}
\icmlauthor{Scott Geng}{uw}

\icmlauthor{Lorenzo Cavallaro}{ucl}
\icmlauthor{Junfeng Yang}{col}
\icmlauthor{Suman Jana}{col}
\end{icmlauthorlist}

\icmlaffiliation{col}{Columbia University}
\icmlaffiliation{chi}{The University of Chicago}
\icmlaffiliation{umich}{University of Michigan}
\icmlaffiliation{uw}{University of Washington}
\icmlaffiliation{ucl}{University College London}
\icmlaffiliation{uiuc}{University of Illinois Urbana-Champaign}

\icmlcorrespondingauthor{Kexin Pei}{kpei@cs.uchicago.edu}
\icmlcorrespondingauthor{Suman Jana}{suman@cs.columbia.edu}


\vskip 0.3in
]

\printAffiliationsAndNotice{\icmlEqualContribution} 

\input{abst}

\input{intro}
\input{prelim}

\input{method}

\input{impl}

\input{setup}

\input{eval}

\input{discussion}

\input{related}

\input{conclusion}

\input{ack}

\input{impact}

\bibliography{paper}
\bibliographystyle{icml2024}

\input{appendix}

\end{document}

%% file: abst.tex
\begin{abstract}

This paper tackles the challenge of teaching code semantics to Large Language Models (LLMs) for program analysis by incorporating code symmetries into the model architecture. We introduce a group-theoretic framework that defines code symmetries as semantics-preserving transformations, where forming a code symmetry group enables precise and efficient reasoning of code semantics. Our solution, \sys, develops a novel variant of self-attention that is provably equivariant to code symmetries from the permutation group defined over the program dependence graph. \sys obtains superior performance on five program analysis tasks, outperforming state-of-the-art code models, including GPT-4, without any pre-training. Our results suggest that code LLMs that encode the code structural prior via the code symmetry group generalize better and faster.

\end{abstract}

%% file: intro.tex
\section{Introduction}
\label{sec:intro}
Automated program analysis using Large Language Models (LLMs) has become widely popular for software engineering and security tasks~\citep{oss_llm_fuzz, didact-google}, but it remains unclear whether code LLMs can stay robust and generalize to new code~\citep{henke2022semantic, rabin2021generalizability, gao2023two, gao2023discrete, yefet2020adversarial, bundt2022black, zhang2023pelican}. 
This paper aims to enhance LLMs by establishing and preserving fundamental code symmetries, drawing inspiration from translation and rotation symmetries that typically hold in vision.



\para{Code symmetry.} 
Intuitively, symmetry of code refers to any transformation applied to a code block that preserves the semantics (\ie input-output behavior) of the original code. 
Consider a (sequential) code fragment \code{x=2;y=4}. Reordering the instructions to \code{y=4;x=2} does not change the semantics of the code. 
Of course, any code analysis task that depends solely on the semantics of the code (\eg bug detection) needs to preserve these symmetries by staying invariant to the transformations.
Otherwise, if a bug detector flips its prediction from correct to buggy due to such simple semantics-preserving permutations, developers will lose confidence in the tool~\cite{coverity:cacm}.
Formally, given a code block $c$ and a set of symmetries $G$, an LLM $m$ should ensure $\forall g \in G, m(g(c))=m(c)$.
Incorporating such an invariant property in the model has proven an effective approach in many domains to enforce domain-specific rules and improve generalization~\cite{cohen2016group}.

\para{Limitations of existing approaches.} 
A popular way to train LLMs to be robust to code symmetries is via large-scale pre-training, where the pre-training dataset likely includes many semantically equivalent code samples~\cite{roziere2023code}. 
While this approach improves the generalization, it is inherently a best-effort attempt and does not guarantee invariance for the pre-trained model.
Without the guarantee, the trained model lacks assurance when deployed for program analysis~\cite{ullah2023can}, \eg the malware can be easily transformed to evade detection.
In fact, we find that state-of-the-art LLMs break desired invariances at an alarmingly high rate, \eg 18\% in CodeLlama in predicting function names (Table~\ref{tab:violation-rate}) even for simple code symmetries like a two-statement permutation.

A more direct approach is to explicitly enumerate the code symmetries via data augmentation.
However, it is prohibitively expensive due to the sheer number of possible symmetries and their compositions. 
Similar to pre-training, even exhausting the code symmetries in the augmented training set does not provide any guarantee that the trained model stays invariant to the observed code symmetries. 
An alternative strategy involves approximating code symmetry structures as priors within the model's architecture, \eg data and control flow graphs, using Graph Neural Nets (GNNs) ~\cite{allamanis2017learning}.
Such approaches are argued to offer a better generalization, \eg as evidenced in Table~\ref{tab:violation-rate} where GGNN has the lowest violation rate among the baselines.
However, graph architectures often restrict a model's expressiveness relative to LLM architectures~\cite{ying2021transformers}.
Moreover, the existing common practice of encoding code structures into GNNs does not explicitly preserve code symmetries and thus lacks guarantee that the code symmetries are indeed preserved (see \S\ref{sec:eval}).




\begin{figure*}[!t]
\begin{minipage}{.22\linewidth}
    
    \centering
    \footnotesize
    \setlength{\tabcolsep}{2pt}
    \renewcommand{\arraystretch}{1}
    
    \captionof{table}{Invariance violation rate across different code models (darker colors indicate more violations).}
    \label{tab:violation-rate}
        \begin{tabular}{ll}
            \toprule[1.1pt]
            & \textbf{Violation} \\ \midrule[.9pt]
            \textbf{\sys (Ours)} & \textbf{0\%} \\ 
            code2vec & \cellcolor{lightred-darkest}61\% \\
            code2seq & \cellcolor{lightred-darkest}52\% \\
            CodeLlama & \cellcolor{lightred-medium}18\% \\
            CodeT5 & \cellcolor{lightred-medium}16\%\\
            DOBF & \cellcolor{lightred-darkest}41\%\\
            GGNN & \cellcolor{lightred-lightest}7\% \\
            GPT-4 &  \cellcolor{lightred-darkest}43\% \\  
            GraphCodeBERT & \cellcolor{lightred-darkest}31\%\\
            WizardCoder & \cellcolor{lightred-medium}14\% \\ 
            \bottomrule[1.1pt]
        \end{tabular}
\end{minipage}\hfill%
\begin{minipage}{.76\linewidth}
    \centering
    
    \subfloat[\textbf{\sys (Ours)}]{
    \includegraphics[width=0.49\linewidth]{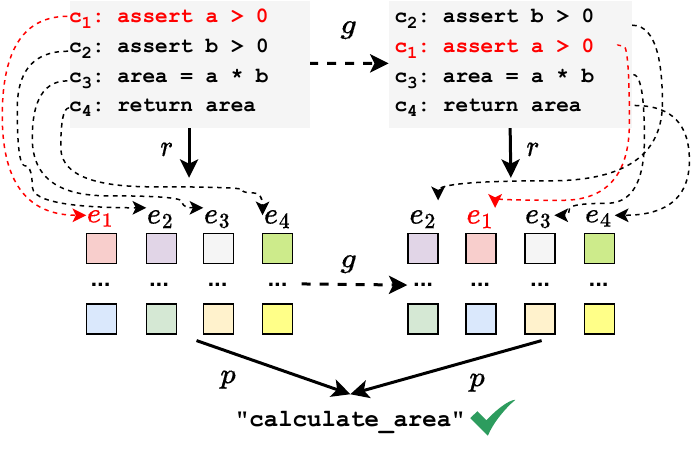}
    \label{subfig:example-g-invariant}}\hspace{-.3cm}%
    \subfloat[Existing code LLMs]{
    \includegraphics[width=0.5\linewidth]{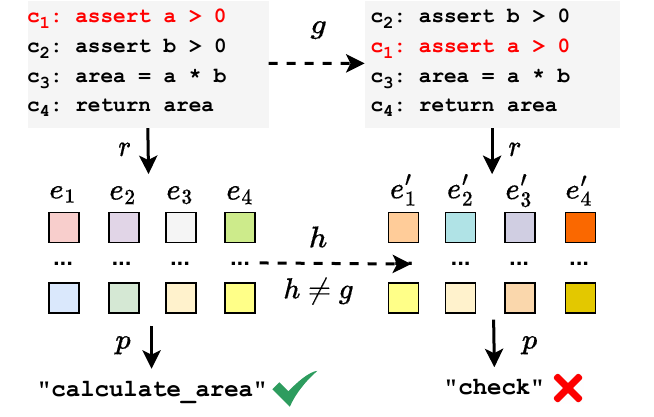}
    \label{subfig:example-g-variant}}

    \caption{
    (a) \sys as a $G$-invariant function name predictor where $G$ is a group of semantics-preserving statement permutations $g$.
    (b) Code LLMs not preserving the symmetries in $G$ and thus mispredict the label.}
    \label{fig:motivating-example}
\end{minipage}

\end{figure*}

\para{Our approach.} In this paper, we investigate how to modify expressive architecture like Transformers to provably impose semantic priors like code symmetries while still preserving the learning capacity.
We introduce a group-theoretic framework to precisely define code symmetries in terms of semantics-preserving statement permutations and create LLM architectures that inherently preserve these symmetries \emph{by construction}. 
Importantly, such a group-theoretic framework expresses semantic priors in a syntax-agnostic way while being amenable to be encoded in neural architectures~\cite{cohen2016group, romero2020group}. 
This makes symmetries an ideal option to represent program semantics as the inductive bias for code LLM architectures. 

Using this framework, we present \sys, a variant of LLM architecture designed to \emph{provably guarantee} invariance to semantics-preserving statement permutations. 
This is achieved through a $G$-equivariant code representation ($r$) followed by a $G$-invariant prediction ($p$), with $G$ determined based on the graph automorphisms of the code block's interpretation graph (a generalization of program dependence graph).


\begin{figure}[!t]
  \begin{center}
    \centering
    
    \subfloat[\textbf{\sys (ours)}]{
    \includegraphics[width=.45\linewidth]{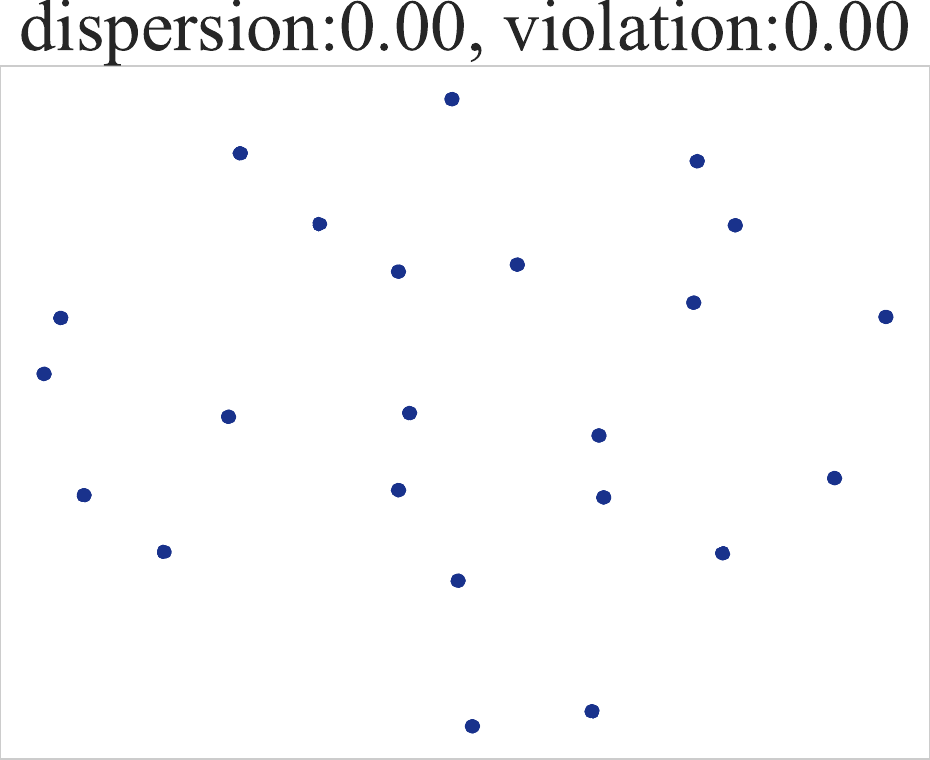}
    \label{subfig:symc-tsne}}
    \subfloat[code2vec]{
    \includegraphics[width=0.45\linewidth]{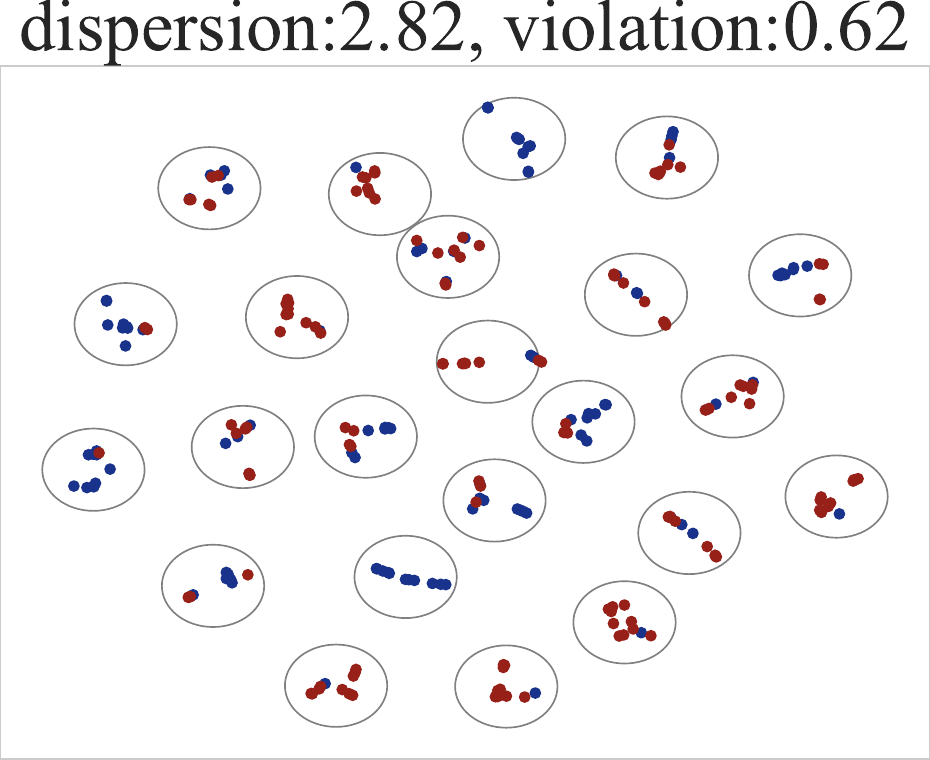}
    \label{subfig:code2vec-tsne}}


    \subfloat[CodeT5]{
    \includegraphics[width=.45\linewidth]{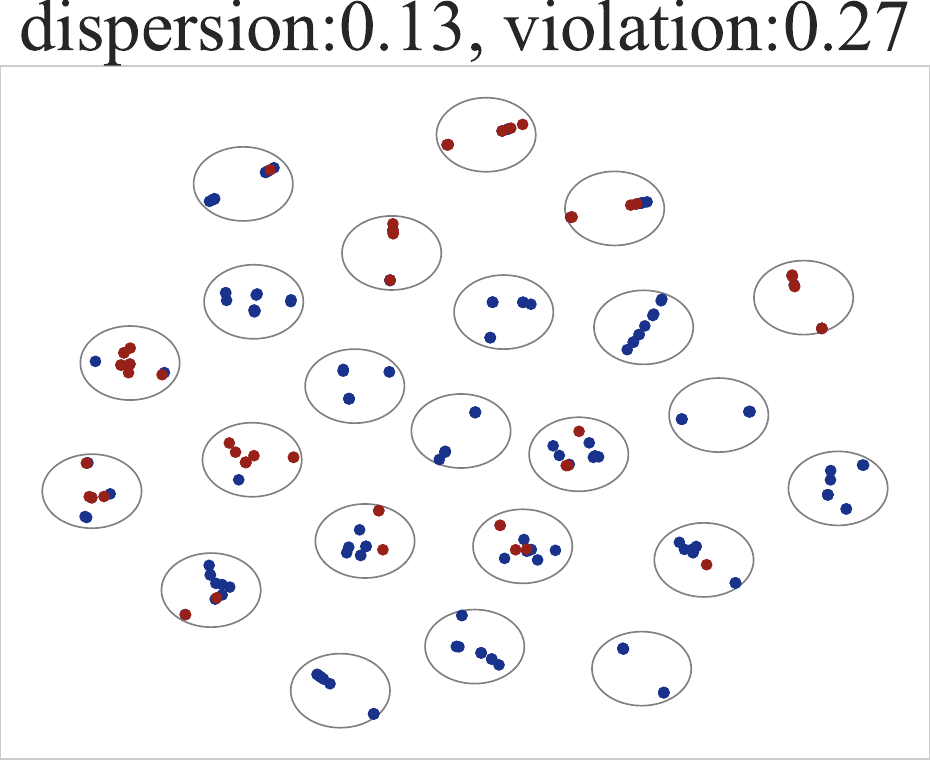}
    \label{subfig:codet5-tsne}}
    \subfloat[CodeLlama-7b]{
    \includegraphics[width=0.45\linewidth]{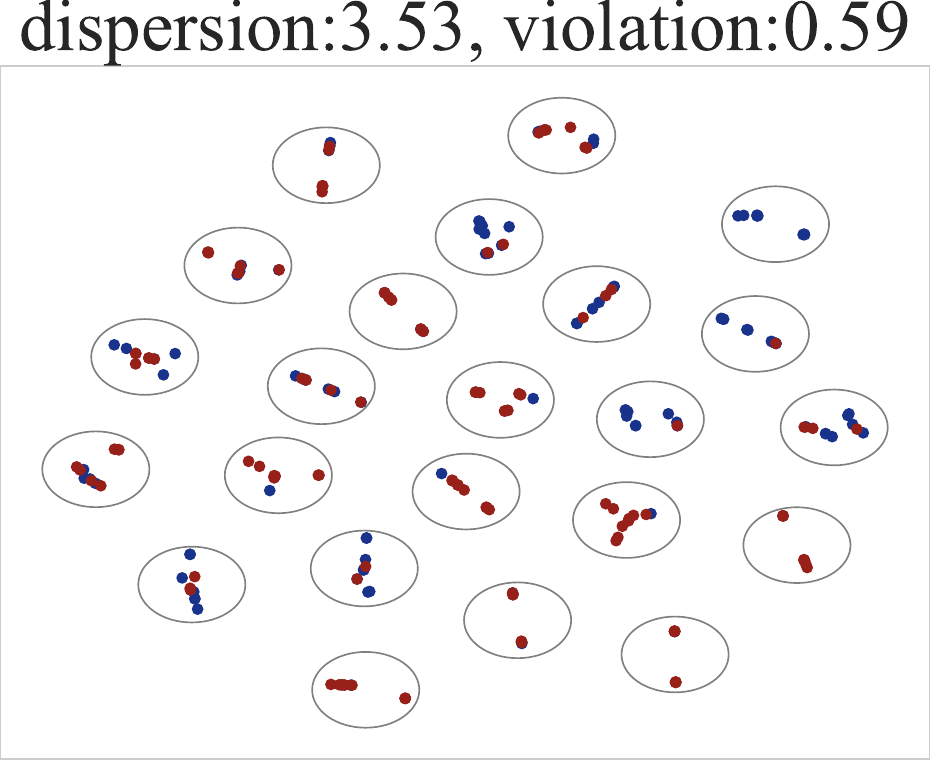}
    \label{subfig:codellama-tsne}}
    
    \caption{Each cluster represents the learned embeddings of a code block and its semantics-preserving permuted versions. The cluster's dispersion (variances to the mean) indicates that the permutation changes the embeddings, while the color turning from \colorbox{lightblue}{blue} to \colorbox{lightred-medium}{red} indicates the changed predictions. We take the mean of \sys's embedding so it becomes permutation-invariant (\S\ref{subsec:ig_invariant_predictive_learning}).}
    
    \label{fig:tsne}
    
  \end{center}
\end{figure}

Figure~\ref{subfig:example-g-invariant} shows a concrete example of the benefit of \sys when deployed for function name prediction.
The code snippets illustrate a semantics-preserving statement reordering. 
\sys enforces its output to stay invariant via keeping its learned representation $G$-equivariant, where the code representation $(e_1,e_2,e_3,e_4)$ is transformed into $(e_2,e_1,e_3,e_4)$, followed by a $G$-invariant prediction module.
By contrast, Figure~\ref{subfig:example-g-variant} shows an existing code model~\citep{jin2022symlm} that does not preserve permutation symmetry. 
In this case, the code representation $(e_1,...,e_4)$ is transformed into a completely different set of embeddings $(e'_1,...,e'_4)$, leading to a changed prediction.
In fact, the t-SNE visualization in Figure~\ref{fig:tsne} shows that the learned code embeddings of existing code models are highly dispersed when the code is under semantics-preserving permutations, and a large fraction of the permuted samples have their labels mispredicted.

\para{Result summary.} We evaluate \sys on five program analysis tasks against 12 source code and binary analysis baselines. 
For semantics-preserving permutations, \sys is guaranteed to stay invariant while the state-of-the-art code LLMs, \eg CodeLlama, violate the invariance by 31.4\% on average.
As a result, \sys outperforms the extensively pre-trained baselines by up to 67.5\%, even though \sys is only trained from scratch \emph{without requiring any pre-training}.
For unseen semantics-preserving source code transformations beyond permutations, \sys surpasses the state-of-the-art code baselines by up to 30.8\%, while maintaining 104.8$\times$ smaller model size.
On sophisticated code transformations introduced by compiler optimizations and obfuscations, \sys outperforms the extensively pre-trained binary analysis baseline by 30.7\%.

%% file: prelim.tex
\section{Preliminaries}
\label{sec:preliminaries}



This section briefly describes the symmetry groups.
See Appendix~\ref{app_sec:preliminaries} for a more formal description. 

A \emph{symmetry group} $(G,\circ)$ consists of a non-empty set $G$ of transformations and a binary operator $\circ:G\times G\rightarrow G$, where $\circ$ operates on two elements in $G$, \eg $x,y\in G$, and produces a new transformation $z=x\circ y, z\in G$.
The binary operator has to be associative, invertible, and there exists an identity $\exists \mathbf{1}\in G, \forall x\in G, x\circ \mathbf{1}=\mathbf{1}\circ x$.


The elements of a $G$ are abstract transformations that become concrete when they \emph{act} on some set $X$, \ie they transform $x\in X$ into $x'\in X$ while keeping some properties of $x$ \emph{invariant}. 
Formally, an action $\bullet$ of a $G$ is a binary operation defined on a set of objects $X$, \ie $\bullet: G\times X\rightarrow X$, where it is also associative and has an identity.



It is common in the group theory literature to use $\circ$ to denote both \emph{action} and \emph{composition}, when it is clear from the context~\citep{higgins2018towards}. 
It is also customary to interchange $g(x)$ and $g\circ x$.
Therefore, we treat $g\bullet (h\bullet x)$, $g\circ (h\circ x)$, and $g(h(x))$ as the same in the rest of this paper. 


A symmetry group comes with two properties, namely \emph{invariance} and \emph{equivariance}, that formalize the concept of preservation of some properties when a set $X$ is acted upon by $G$. 
Let $f$ be a function that maps each element $x\in X$ to a corresponding element $y$ in the set $Y$, indicating the property's value for that particular element.
$f$ is called \textbf{$G$-invariant} if $\forall g\in G, \forall x\in X, f(g\circ x)=f(x)$.
$f$ is called \textbf{$G$-equivariant} if $\forall g\in G, \forall x\in X, f(g\circ x)=g\circ f(x)$.








%% file: method.tex
\section{Method}
\label{sec:method}


This section describes the construction of group-equivariant self-attention layers and the group-invariant code model.

\subsection{Invariance \& Equivariance for Code Models}
\label{subsec:representation_learning}

\para{Code representation units.}
We establish formal definitions of the code space as a collection of code blocks, which serve as the input space for representation learning. 

\begin{definition}
\label{def:pru}

A {\bf code representation unit} (\eg procedure) $c$ consists of $n$ instructions from an instruction set $I$, \ie $c\in I^n$. 
The {\bf code space} $I^n$ is the set of all code representation units of interest.
\end{definition}

A typical Code Representation Unit (CRU) is a method with well-defined interfaces, ensuring controlled interaction with other methods, without arbitrary control transfers. 
Below, we provide formal definitions for learning program representation and predictive learning.

We establish formal properties for code analysis models with explicit representation learning $r$ and predictive learning $p$ based on $G$-equivariance/invariance. 

\begin{definition}[{\bf $G$-equivariant code representation learning}]
\label{def:g-equivariant-r}

Let $G$ be a symmetry group consisting of \emph{semantics-preserving transformations} applied to a CRU $c\in I^n$. A representation function $r:I^n \rightarrow \mathbb{R}^{d\times n}$ is $G$-equivariant if for every $g\in G$ and $c\in I^n$, we have $g\circ r(c)=r(g\circ c)$ ($d$ denotes the dimension of the embedding to which each instruction is mapped).

\end{definition}

Note that here the input space of $r$ ($I^n$) and its output space ($\mathbb{R}^{d\times n}$) are both sets of size $n$, where each instruction $I$ is mapped to a $R^d$ vector by the representation function. This consideration is necessary to ensure the symmetry group can act on both sets appropriately.

\begin{definition}[{\bf $G$-invariant code predictive learning}]
\label{def:g-invariant-p}

Let $G$ be a symmetry group consisting of \emph{semantics-preserving transformations} applied to program representation vector $c\in I^n$. A predictive learning function $p: \mathbb{R}^{d\times n} \rightarrow \mathbb{R}^{L}$ is $G$-invariant if $\forall g\in G, \forall e\in \mathbb{R}^{d\times n}$, $p(g\circ e)=p(e)$.
\end{definition}

Stacking $p$ on top of $r$, $p\circ r$, leads to a $G$-invariant model according to Lemma~\ref{lemma:g_equivariant_and_invariant}.

\subsection{Semantics-Preserving Program Symmetries} 
\label{subsec:semantic_preserving_program_symmetries}

A code symmetry is a program transformation that preserves the input-output behavior of a CRU when interpreted by the program interpretation function $f$. 
The program interpretation function takes a CRU $c \in I^n$ as input.
We use $\mathcal{I}$ to represent the set of all input values to execute CRU, and produce output values represented by the set $\mathcal{O}$.


\begin{definition}
\label{def:program_symmetry}
A {\bf semantics-preserving code symmetry} $g$ is a transformation acting on $c\in I^n$ ($g:I^n \rightarrow I^n$) such that $\forall in\in \mathcal{I}, \forall out\in \mathcal{O}, f(g\circ c, in)=f(c, in)=out$.
\end{definition}

\begin{definition}
\label{def:program_symmetry-group}
A semantics-preserving {\bf program symmetry group} $G$ is a set of semantics-preserving program symmetries that also satisfy the group axioms.
\end{definition}



\subsection{$Aut(\mathcal{IG})$: A Program Symmetry Group}
\label{subsec:finding_program_symmetry}

In this paper, we focus on a specific symmetry group that maintains the structural integrity of CRUs by utilizing their inherent compositional structure. 
However, note that this approach is not the only way to form code symmetry groups and does not encompass all possible code symmetries. 
We leave further exploration in these directions to future work.
Next, we describe the compositional structure of the program interpreter $f$ operating on a CRU, enabling us to define the program interpretation graph that links CRUs to their input-output behavior. 

\para{Compositional structure of program interpreter $f$.}
The interpreter function $f$ (defined in \S\ref{subsec:semantic_preserving_program_symmetries}) can be represented as a composition of individual per-instruction interpreter functions $\{f_1,...,f_n\}$.
Each $f_i: \mathcal{I}_i\rightarrow \mathcal{O}_i$ interprets a single instruction $c_i$ from the instruction set $I$ (Definition~\ref{def:pru}), takes the input values $in_i\in\mathcal{I}_i$, and produce the output values $out_i\in\mathcal{O}_i$. 
The output of $f_i$ can include both data flow elements (\eg variables or memory locations with values assigned by $f_i$) and control flow elements (\eg addresses of next interpreter functions $f_j\in f$ assigned by $f_i$). 
Consequently, we can express $f$ as the composition of different individual interpreters, \ie $f_n\circ...\circ f_1$, where later instructions act on the output of previous instructions. 



\para{Program interpretation graph ($\mathcal{IG}$).}
Programs often involve different control flow paths, such as if-else statements, leading compositions between individual interpreter functions to a directed graph instead of a linear sequence. This graph is referred to as the program interpretation graph. For a given CRU $c$, there can be multiple execution paths, each exercising different subsets of $\{f_1,...,f_n\}$.

To construct the interpretation graph $\mathcal{IG}=(V,E)$, we consider all feasible execution paths of $c$. 
In $\mathcal{IG}$, each node $V_i\in V$ corresponds to $f_i$, and each directed edge $E_{i,j}\in E$ (connecting $V_i$ to $V_j$) represents at least one execution path where $f_j$ takes the output of $f_i$ as input, \ie $E_{i,j}=(out_i, in_j)$. 


\para{Automorphism group of interpretation graph.}
Our objective is to find a group of symmetries that act on $c$ while preserving its input and output behavior as interpreted by $f$ in terms of $\mathcal{I}$ and $\mathcal{O}$ (Definition~\ref{def:program_symmetry}). 
Intuitively, as $\mathcal{IG}$ represents all execution paths of $c$, any transformations that preserve $\mathcal{IG}$ should also preserve the execution behavior of $c$.
Therefore, we aim to uncover a group of symmetries that preserve $\mathcal{IG}$ (Theorem~\ref{theorem:automorphism}), and such a group can guide us to construct code analysis model that can stay invariant to all symmetries of the group (\S\ref{subsec:ig_invariant_code_analysis}).


To achieve this, we consider a specific set of symmetries called the \emph{automorphisms} of $\mathcal{IG}$, denoted as $Aut(\mathcal{IG})$. 
An automorphism is a group of symmetries $\sigma\in Aut(\mathcal{IG})$ that act on the interpretation graph $\mathcal{IG}=(V,E)$. 
Intuitively, graph automorphisms can be thought of as permutations of nodes that do not change the connectivity of the graph. 
$Aut(\mathcal{IG})$ is formally defined as follows:

\begin{definition}[$\mathcal{IG}$ Automorphism]
\label{def:ig_automorphism}

$\mathcal{IG}$ automorphism is a group of symmetries $\sigma\in Aut(\mathcal{IG})$ acting on an interpretation graph $\mathcal{IG}=(V,E)$, where $\sigma$ is a bijective mapping: $\sigma: V \rightarrow V$, such that for every edge $E_{i,j} \in E$, \ie connecting $f_i$ and $f_j$, there is a corresponding edge $(\sigma(f_i), \sigma(f_j)) \in E$.

\end{definition}

We now show how the automorphism $\sigma\in Aut(\mathcal{IG})$ preserves all input and output behavior of $\{f_1,...,f_n\}$ in the space of $\mathcal{I}$ and $\mathcal{O}$.
As mentioned earlier, graph automorphism is a permutation on the set of nodes in $\mathcal{IG}$ such that the edges $E_{i,j}=(out_i,in_j)$ are preserved in the transformed $\mathcal{IG}'$.
As each $f_i\in\{f_1,..,f_n\}$ operates on $c_i\in c$, we have the following (see Appendix~\ref{app:proof} for the proof):

\begin{theorem}
\label{theorem:automorphism}

The set of automorphisms $\sigma\in Aut(\mathcal{IG})$ forms a program symmetry group.
    
\end{theorem}

\subsection{$Aut(\mathcal{IG})$-Equivariant Code Representation}
\label{subsec:ig_invariant_code_analysis}


Existing program analyses using Transformer typically involve an embedding layer followed by applying $l$ self-attention layers $A^l$. 
A prediction head $F$ is then placed on top of $A^l$ for downstream analysis tasks.
We can thus consider the representation learning $r$ as the composition of the embedding layer and $A^l$, with $F$ as the predictive learning $p$ (\S\ref{subsec:representation_learning}).
We now present the development of a new self-attention layer that is $Aut(\mathcal{IG})$-equivariant. 





\noindent
{\bf Self-attention.} The standard self-attention computation can be succinctly represented as $w_v\cdot s(w_k^T\cdot w_q)$, where $w_v$, $w_k$, and $w_q$ are learnable parameters for transforming value, key, and query, respectively, and $s(\cdot)$ represents scaling by $\sqrt{d}$ and applying Softmax (see Appendix~\ref{app_sec:preliminaries}).

It is easy to show that the existing self-attention layer is equivariant to permutations (Appendix~\ref{app:proof}). 
However, we want to make the self-attention layers equivariant \emph{only} to $Aut(\mathcal{IG})$, not \emph{all permutations}. 
In the following, we describe how to build $Aut(\mathcal{IG})$-equivariant self-attention.

\para{Biasing self-attention with a distance matrix.}
To build $Aut(\mathcal{IG})$-equivariant self-attention layers, denoted as $G\!A$, we add a customized distance matrix $d_{\mathcal{IG}}$ to $G\!A$: $G\!A(e) = w_v e\cdot s(w_k e^T\circ w_q e+d_{\mathcal{IG}})$.
Such a distance matrix is a superset of the adjacency matrix of $\mathcal{IG}$, encoding a richer topology structure of the graph.
We relax the definition of distance matrix $d_{\mathcal{IG}}$ here to be no longer symmetrical, as long as it satisfies the following two properties:
(1) $d_{\mathcal{IG}}$ stays invariant when $\sigma\in Aut(\mathcal{IG})$ acts on $\mathcal{IG}$: $d_{\mathcal{IG}}=\sigma(d_{\mathcal{IG}})$, and 
(2) $d_{\mathcal{IG}}$ commutes with permutation matrix $p_{\sigma}$ ($\sigma\in Aut(\mathcal{IG})$).

We will describe a concrete instantiation of $d_{\mathcal{IG}}$ in \S\ref{subsec:encoding_graph_topology}.
Based on the two properties, we have the following Theorem (with its proof in Appendix~\ref{app:proof}.

\begin{theorem}
\label{theorem:ga_is_equivariant}
    Self-attention $G\!A(e) = w_v e\cdot s(w_k e^T\cdot w_q e+d_{\mathcal{IG}})$ is $Aut(\mathcal{IG})$-equivariant.
\end{theorem}


As the embedding layer is trivially $Aut(\mathcal{IG})$-equivariant, composing it with $Aut(\mathcal{IG})$-equivariant self-attention layers remains $Aut(\mathcal{IG})$-equivariant (Lemma~\ref{lemma:two_g_equivariant}).

\subsection{$Aut(\mathcal{IG})$-Invariant Predictor}
\label{subsec:ig_invariant_predictive_learning}

We describe two prediction modules that are inherently $Aut(\mathcal{IG})$-invariant, so stacking them on top of the $Aut(\mathcal{IG})$-equivariant self-attention layers leads to an $Aut(\mathcal{IG})$-invariant code model (Lemma~\ref{lemma:g_equivariant_and_invariant}).

\para{Token-level.} 
Token-level predictor is often employed when each input token needs a label, \eg predicting memory region per instruction (\S\ref{sec:experimental_setup}).
As the automorphism acts on the input sequence $e$ but not individual tokens, \ie the value of the embedding vectors, the automorphism $\sigma$ does not apply to the query vector $q_i$ (\S\ref{subsec:ig_invariant_code_analysis}).
Therefore, we have Lemma~\ref{lemma:ig_invariant_token_predictor}. 
See Appendix~\ref{app:proof} for complete proof.

\begin{lemma}
\label{lemma:ig_invariant_token_predictor}
    The biased self-attention computing the embedding $e'_i=GA(e_i)$ is $Aut(\mathcal{IG})$-invariant.
\end{lemma}




\para{Pooling-based.}
Another popular $Aut(\mathcal{IG})$-invariant predictor involves pooling the embedding sequence $e'=GA(e)$, \eg using max or mean.
Pooling operators are invariant to permutations, thus to $Aut(\mathcal{IG})$, \eg the mean pooling $\mu(e')=(\Sigma^n_{i=1}e'_i)/n$ is not sensitive to the order of $(e'_1,...,e'_n)$. 
Pooling-based predictor is often employed when we aim to predict the property for the entire input sequence, \eg predicting the function signature, detecting function similarity, etc. (\S\ref{sec:experimental_setup}).

%% file: impl.tex
\section{\sys Implementation}
\label{sec:impl}

This section elaborates on the design choices to implement $Aut(\mathcal{IG})$-invariant code analysis.
Figure~\ref{fig:symc-arch} shows the simplified steps of biasing the self-attention layers to be equivariant to the semantics-preserving permutation group.

\begin{figure}[!t]
    \centering
    \includegraphics[width=\linewidth]{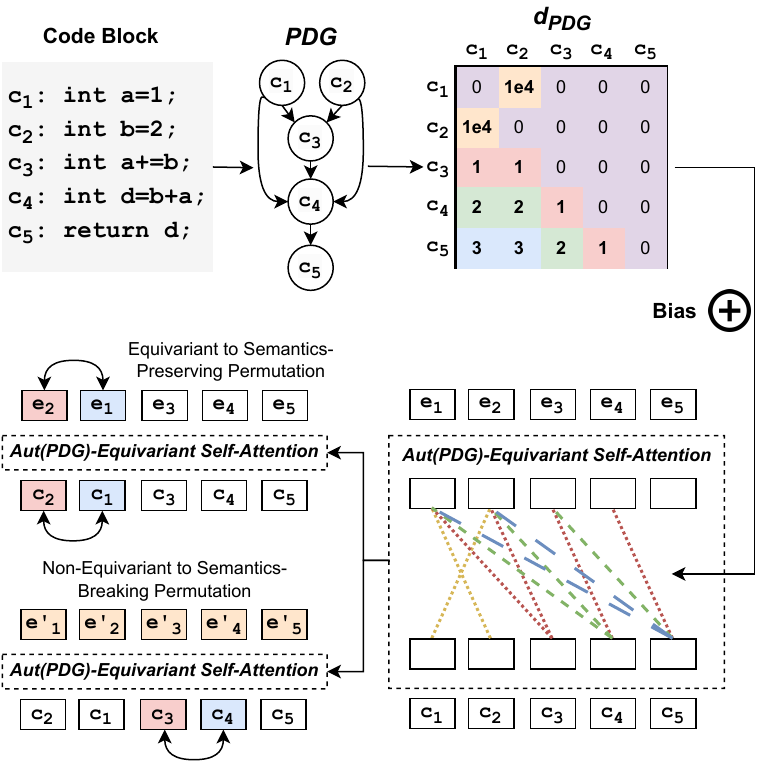}
    \caption{Simplified \sys architecture, which takes as input the code block and its program dependence graph (PDG) to construct the $Aut(P\!D\!G)$-equivariant self-attention head.}

    
    \label{fig:symc-arch}
\end{figure}


\subsection{Relaxing $\mathcal{IG}$ to Program Dependence Graph}
\label{subsec:pdg_symmetry}

\S\ref{subsec:ig_invariant_code_analysis} demonstrated how to build $Aut(\mathcal{IG})$-equivariant self-attention layers. 
However, directly constructing $\mathcal{IG}$ is computationally impractical as we need to enumerate all possible execution paths. 
To address this, we consider \emph{program dependence graph} (PDG), a sound over-approximation to $\mathcal{IG}$ that explicitly captures the control/data dependencies and can be computed statically and efficiently.
PDG ($V_{PDG}, E_{PDG}$) is a super graph of $\mathcal{IG}$ due to the conservative data flow analysis, sharing the same vertices but having a superset of edges ($E_{PDG} \supseteq E_{\mathcal{IG}}$). 
Enforcing PDG to be a super graph of $\mathcal{IG}$ is crucial because a subgraph's automorphism group is a subgroup of the super graph's automorphism group ($Aut(P\!D\!G) \supseteq Aut(\mathcal{IG})$). 
Thus, if the self-attention layer is $Aut(P\!D\!G)$-equivariant, it is guaranteed to be $Aut(\mathcal{IG})$-equivariant.

\para{PDG construction.} We construct PDG edges based on data and control dependencies between instructions. We consider three types of data dependencies: read-after-write, write-after-read, and write-after-write. Additionally, control dependencies are included to determine the execution order. These dependencies create a partial order of instructions, preventing permutations that would violate the edge directions and potentially alter the program's input-output behavior.(\S\ref{subsec:finding_program_symmetry}).

To identify the control and data dependencies, we employ a computationally efficient, conservative static analysis approach. 
For instance, we treat any two statements that access memory as dependent. 
However, this conservative approach may overlook potential symmetries. We aim to integrate more accurate analyses in the future, such as alias analysis through abstract interpretation or dynamic analysis, to reduce overapproximation while managing the trade-off of increased overhead. 
Interleaving standard dependency analysis with training/inference to minimize analysis overhead presents a promising direction for future exploration.

Implementing the static analysis to identify symmetries presents an additional challenge, as symmetries are defined at the level of tokens as seen by LLMs, whereas compilers and other program analysis tools typically perform data dependency analysis at the Intermediate Representation (IR) level. 
It requires significant engineering effort to correlate IR-level dependencies with source-level tokens and subtokens for the self-attention layers. 
Hence, for our initial prototype, 
we developed our analysis method based on the source-level ASTs (as described in \S\ref{app_subsec:implementation_details}) to simplify the mapping from the analysis results back to the source code tokens and subtokens. 
We note that integrating more sophisticated dependency analysis routines from existing compiler frameworks to identify symmetries is a promising area for future work.


\subsection{Encoding Graph Structure}
\label{subsec:encoding_graph_topology}

This section presents a concrete instance of the distance matrix defined on PDG, which enables us to prove $Aut(P\!D\!G)$-equivariance for the resulting self-attention layers.

\para{Distance matrix.} Let $d$ denote the distance matrix of PDG where $d_{ij}$ represents the distance between nodes $V_i$ and $V_j$. 
Each entry $d_{ij}$ is a 2-value tuple $(p_{ij}, n_{ij})$, indicating the longest path from the lowest common ancestor of $V_i$ and $V_j$, denoted as $T_{ij}$, to $V_i$ and $V_j$, respectively. 
We call $p_{ij}$ the positive distance and $n_{ij}$ the negative distance.

We incorporate $d$ into the Multi-Head Self-Attention (MHA) to ensure $Aut(P\!D\!G)$-equivariance with specific modifications to the attention heads to handle positive and negative distances.
Particularly, the first half of the attention heads $M\!H\!A^i(e)$, for $i\in [1,h/2]$, are combined with the matrix $dp$ formed by the positive distances in $d$ (denoted as $dp_{ij}=p_{ij}$). The second half of the attention heads $M\!H\!A^i(e)$, for $i\in [h/2+1,h]$, are combined with the matrix $dn$ formed by the negative distances in $d$ (denoted as $dn_{ij}=n_{ij}$). The modified attention heads are defined as:
(1) $M\!H\!A^i(e) = w_v e\cdot s(w_k e^T\cdot w_q e+dp), i\in [1,h/2]$, 
(2) $M\!H\!A^i(e) = w_v e\cdot s(w_k e^T\cdot w_q e+dn), i\in [h/2+1,h]$.

It is easy to show $d$ satisfies the two properties defined in \S\ref{subsec:ig_invariant_code_analysis} (see Appendix~\ref{app:proof}).
We thus have:

\begin{lemma}
\label{lemma:invariant_distance_matrix}
The distance matrix $d$ of PDG remains invariant under the action of $\sigma\in Aut(P\!D\!G)$.
\end{lemma}

\begin{lemma}
\label{lemma:comutative_distance_matrix}
The distance matrix $d$ of PDG commutes with permutation matrix $p_{\sigma}$ of the automorphism $\sigma\in Aut(P\!D\!G)$: $d\cdot p_{\sigma}=p_{\sigma}\cdot d$.
\end{lemma}

Based on these two properties, we can prove each head in MHA is $Aut(P\!D\!G)$-equivariant, following the same proof steps to Theorem~\ref{theorem:ga_is_equivariant}.
Therefore, according to Lemma~\ref{lemma:two_g_equivariant}, MHA composed by multiple $Aut(P\!D\!G)$-equivariant heads is also $Aut(P\!D\!G)$-equivariant.

%% file: setup.tex
\section{Experimental Setup}
\label{sec:experimental_setup}


This section briefly describes the evaluation tasks and code transformations.
We put the detailed description, \eg baselines, evaluation dataset and metrics, etc., in Appendix~\ref{app:detailed_setup}.

\para{Task selection.}
We consider \emph{program analysis} tasks that take as input the program code and output different program properties (see below).
These tasks require comprehending code semantics and behavior, so the model performing these tasks is expected to stay invariant to code symmetries.
We \emph{do not consider code generation tasks}, where the input to the program is natural language or other formal specifications.
We leave the study of enforcing code symmetry constraints in the output as the future work.

Specifically, we consider (1) \emph{function name prediction}, which performs an ``extreme summarization'' of the function behavior. 
(2) \emph{defect prediction}, which detects whether a code block has an error.
(3) \emph{function similarity detection}, which predicts if a pair of functions are semantically similar;
(4) \emph{function signature prediction}, which predicts the types (\texttt{int}, \texttt{float}, etc.) of function arguments;
and (5) \emph{memory region prediction}, which predicts the memory region (stack, heap, etc.) that each memory-accessing instruction can possibly access.
For (3)-(5), we focus on analyzing stripped binaries considering is broad applications in security, \eg vulnerability detection and security retrofitting.

\para{Code transformations.}
We consider a set of real-world semantics-preserving transformations beyond PDG automorphisms to evaluate \sys's generalization by staying $Aut(P\!D\!G)$-equivariant.
Instruction permutation occasionally forms the basis for these transformations.
In particular, we consider two categories of binary code transformations: (1) \emph{compiler optimizations} from GCC-7.5 and Clang-8, some of which reorder instructions for scheduling purposes (\texttt{-fdelayed-branch}, \texttt{-fschedule-insns}); and (2) \emph{compiler-based obfuscations}, where we consider 5 obfuscations following \citet{jin2022symlm}, such as control flow flattening, indirect branching, etc.

In addition to binary transformations, we consider six source code transformations following \citet{rabin2021generalizability} and \citet{wang2022recode}:
\emph{variable rename} (VR) - changes the identifier names;
\emph{statement permute} (SP) -- semantics-preserving permutations;
\emph{loop exchange} (LX) -- switches \code{for} and \code{while};
\emph{boolean exchange} (BX) -- flips the boolean variables and negates all their uses by tracking their def-use chain;
\emph{unused statement} (US) -- injects unused string declaration into a randomly chosen basic block;
\emph{switch to if} (SI) -- transforms \code{switch} from/to \code{if} statements.

%% file: eval.tex
\section{Evaluation}
\label{sec:eval}

\begin{table}[!t]

\setlength{\tabcolsep}{5pt}
\renewcommand{\arraystretch}{1}

\caption{Comparing \sys to baselines against semantics-preserving permutations. We include the F1 score for both prediction tasks before and after the testing samples are permuted. The violation rate measures how many samples get their labels changed after permutation. }


\label{tab:unseen-permute}

\begin{center}
\begin{tabular}{lllll}
\toprule[1.1pt]
{\bf Model} & {\bf Size} & {\bf Before} & {\bf After} & {\bf Violate} \\ \midrule[.9pt]
\multicolumn{5}{l}{\it Function Name Prediction} \\ \hdashline[3pt/5pt]
\textbf{\sys (ours)} & \textbf{68.4M} & \textbf{36.3} &  \textbf{36.3} & \textbf{0\%}  \\
code2seq &  6.3M   & 25.5 & 24.7  & \cellcolor{lightred-darkest}61\% \\
code2vec &   348M  & 17.7 & 19.6 & \cellcolor{lightred-darkest}52\% \\
CodeLlama & 7B & 31.7 & 31.4  & \cellcolor{lightred-medium}18\% \\
CodeT5    & 770M & 25.4 & 25.4  & \cellcolor{lightred-medium}16\%\\
DOBF    & 428M & 16.3 & 20.1  & \cellcolor{lightred-darkest}41\%\\
GGNN     &  53M   & 1.6 & 1.6 & \cellcolor{lightred-lightest}7\% \\
GPT-4    & N/A & 30.3 & 30.7  & \cellcolor{lightred-darkest}43\%\\
GraphCodeBERT    & 481M & 20.8 & 20.6  & \cellcolor{lightred-darkest}31\%\\
WizardCoder & 3B & 33.9 & 34.6  & \cellcolor{lightred-medium}14\% \\ \midrule[.9pt]
\multicolumn{5}{l}{\it Defect Prediction} \\ \hdashline[3pt/5pt]
\textbf{\sys (Ours)} & \textbf{67.7M} & \textbf{68.8} & \textbf{68.8} & \textbf{0\%} \\
CodeBERT & 476M & 62.2 & 61.7 & \cellcolor{lightred-medium}4.1\% \\
CodeLlama & 7B & 51.03 & 51.39 & \cellcolor{lightred-medium}3.4\%\\
CodeT5 & 770M & 63.3 & 60  & \cellcolor{lightred-darkest}6\% \\
DOBF & 428M & 62.4 & 61.5  & \cellcolor{lightred-darkest}2.7\% \\
GPT-4    & N/A & 51.56 & 50  & \cellcolor{lightred-darkest}13.5\%\\
GraphCodeBERT & 481M & 61.7 & 61.7  & \cellcolor{lightred-lightest}1.3\% \\
UnixCoder & 504M & 67.1 & 67.1  & \cellcolor{lightred-medium}2.9\% \\ 
WizardCoder & 3B & 49.24 & 49.24  & 6.8\% \\ 
\bottomrule[1.1pt]
 
\end{tabular}
\end{center}
\end{table}

\subsection{Invariance and Generalization}
\label{subsec:rq1-rq2}

\para{Evaluating $Aut(P\!D\!G)$-invariance.}
As \sys is provably invariant to $Aut(\!P\!D\!G)$, we aim to study how other baselines perform under varying percentages of semantics-preserving statement permutations. 
Table~\ref{tab:unseen-permute} shows that all baselines, even having much larger model sizes and extensively pre-trained, are susceptible to slight permutations, \ie with their prediction changed by 31.4\% on average for function name prediction.
On defect prediction, \sys outperforms the pre-trained baselines by 11.4\% on average; even we already fine-tuned the open-sourced ones, \eg CodeT5, using the same samples and steps as \sys, while \sys is trained from scratch.
Table~\ref{tab:unseen-permute-binary} shows \sys outperforms the state-of-the-art baseline, PalmTree, by 67.5\% and remains robust across all semantics-preserving permutations.

\begin{table*}[!t]

\setlength{\tabcolsep}{8pt}
\renewcommand{\arraystretch}{1}

\caption{Comparing \sys to binary analysis baselines against semantics-preserving permutations. Function signature and memory region prediction are measured in F1. Function similarity detection is measured using AUC (area under the ROC curve). Similar to Table~\ref{tab:unseen-permute}, the magnitude of the violation rate is highlighted in \colorbox{lightred-medium}{red}. }


\label{tab:unseen-permute-binary}

\begin{center}
\begin{tabular}{l|lll|lll|lll}
\toprule[1.1pt]
\multirow{2}{*}{\bf Model} & \multicolumn{3}{c|}{\bf Function Signature} & \multicolumn{3}{c|}{\bf Memory Region} & \multicolumn{3}{c}{\bf Function Similarity} \\ 
 & Before & After & Violate & Before & After & Violate & Before & After & Violate  \\ \midrule[.9pt]
\textbf{\sys} & \textbf{0.88} & \textbf{0.88} & \textbf{0\%} & \textbf{0.86} & \textbf{0.86} & \textbf{0\%} & \textbf{0.96} & \textbf{0.96} & \textbf{0\%} \\
\pto & 0.59 & 0.41 & \cellcolor{lightred-darkest}24\% & 0.57 & 0.43 & \cellcolor{lightred-medium}18\% & 0.72 & 0.69  & \cellcolor{lightred-darkest}31\% \\
\pts & 0.49 & 0.41 & \cellcolor{lightred-lightest}6\% & 0.57 & 0.44 & \cellcolor{lightred-medium}11\% & 0.8 & 0.72  & \cellcolor{lightred-darkest}35\% \\
\ptu & 0.19 & 0.41 & \cellcolor{lightred-darkest}86\% & 0.32 & 0.2 & \cellcolor{lightred-darkest}32\% & 0.71 & 0.72  & \cellcolor{lightred-darkest}38\% \\ \bottomrule[1.1pt]
 
\end{tabular}
\end{center}
\end{table*}

\para{Generalization to unseen transformations.}
Figure~\ref{fig:unseen-source} shows that \sys's performance on new samples introduced by unseen semantics-preserving transformations beyond permutations, outperforming the model specialized for this task, \eg code2seq, by 30.8\%.
It also outperforms the two LLMs, GPT-4 and WizardCoder, by 16.1\% and 1.63\%, respectively.
However, we observe that Code Llama outperforms \sys by 10\%, although \sys achieves better results on statement permutation (SP).
We attribute this to Code Llama's better understanding of natural language due to its extensive pre-training on both code and text, which is especially beneficial for function name prediction. 
Nonetheless, \sys retains the edge of having a much smaller model size (104.8$\times$) without the need of any pre-training.

Besides the source-level transformations, we compare \sys to baselines on more sophisticated code transformations introduced by unseen compiler optimizations and obfuscations. 
Figure~\ref{fig:unseen-opt-obf} shows that \sys outperforms \pto (see Appendix~\ref{app:detailed_setup}) across all binary analysis tasks (we exclude memory region prediction as the dataset does not have this categorization) by 33.8\% and 30.7\% on seen and unseen transformations, respectively. 
While the compiler optimizations and obfuscations often involve more sophisticated transformations not directly related to instruction permutations, \sys maintains its superior generalization.

\begin{figure}[!t]
    \centering
    
    \includegraphics[width=.9\linewidth]{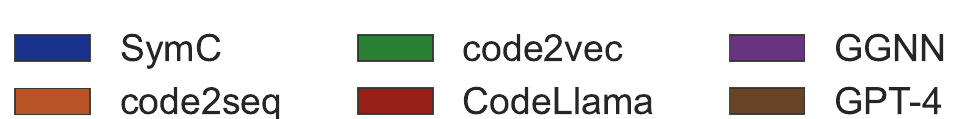}
    
    \includegraphics[width=\linewidth]{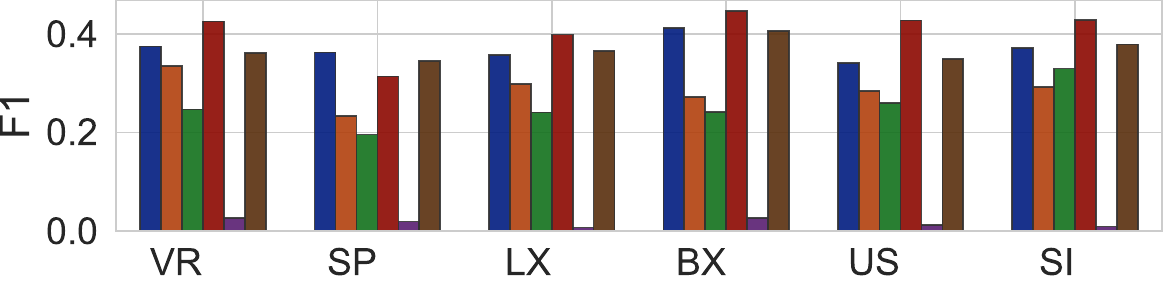}
    \caption{The performance (F1) of \sys and baselines against different unseen code transformations defined in \S\ref{sec:experimental_setup}.}

    
    \label{fig:unseen-source}
\end{figure}

\begin{figure*}[!t]
\centering

\includegraphics[width=.7\linewidth]{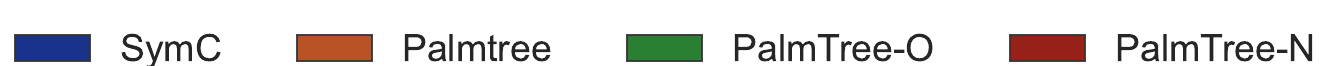}

\subfloat[Cross-OPT similarity]{
\includegraphics[width=0.24\linewidth]{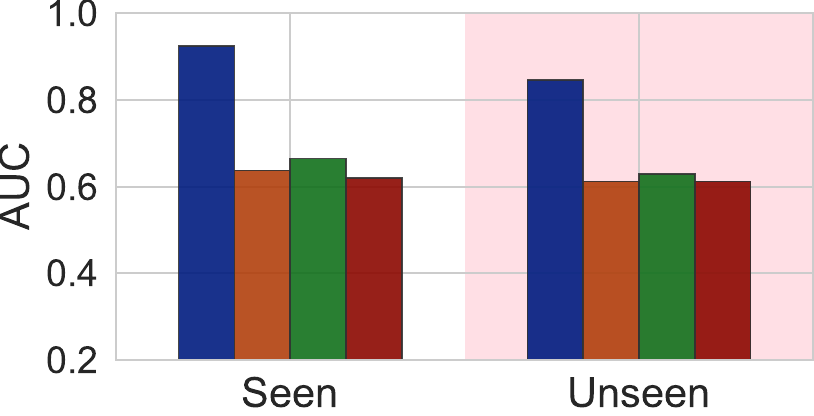}
\label{subfig:opt-sim}}
\subfloat[Cross-OPT signature]{
\includegraphics[width=0.24\linewidth]{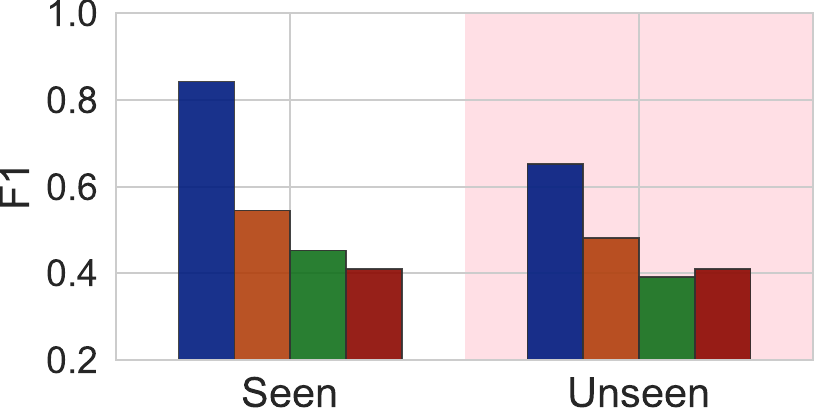}
\label{subfig:opt-sig}}
\subfloat[Cross-OBF similarity]{
\includegraphics[width=0.24\linewidth]{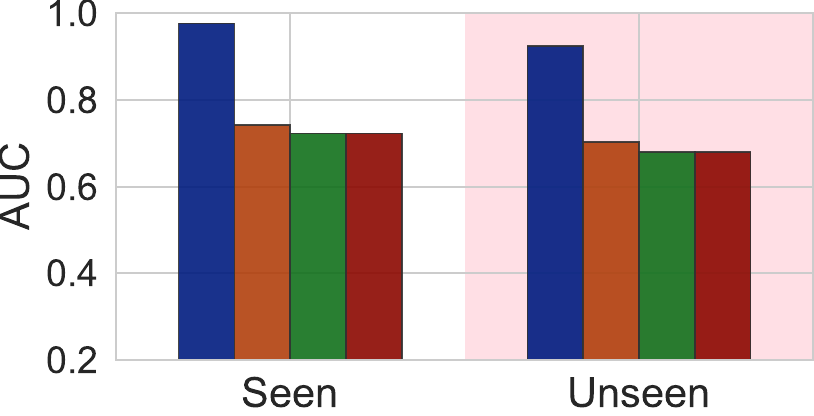}
\label{subfig:obf-sim}}
\subfloat[Cross-OBF signature]{
\includegraphics[width=0.24\linewidth]{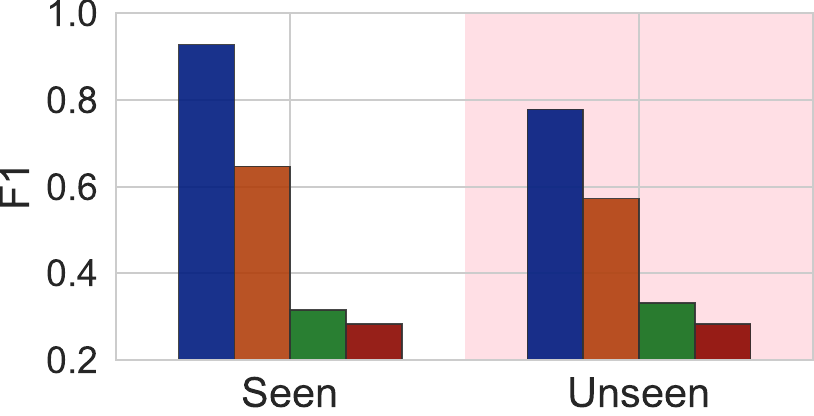}
\label{subfig:obf-sig}}

\caption{Evaluation on unseen optimization and obfuscation (marked in \colorbox{lightred-medium}{pink}). We also include the testing results on \emph{seen} optimizations and obfuscations (but the testing samples are non-overlapping with the training) on the left. }

\label{fig:unseen-opt-obf}
\end{figure*}






\subsection{Training Efficiency}
\label{subsec:rq3}

Besides the improved robustness and generalization, \sys is efficient in avoiding expensive training efforts, \eg some may take up to 10 days~\citep{jin2022symlm}. 
As shown in \S\ref{subsec:rq1-rq2}, \sys, without any pre-training, outperforms the pre-trained baselines.

    
    
    

Therefore, in this section, we aim to study \sys's performance under the limited training resources.
Figure~\ref{fig:resource-constraints} shows that \sys's performance (on memory region prediction) remains the highest when we reduce the model size and training iterations, outperforming \pto by 36.9\% and 21.4\%, respectively.
Even in the most strict scenario, \sys remains 38.2\% and 15.3\% better in both settings.

\begin{figure}[!t]
        \centering
        \includegraphics[width=\linewidth]{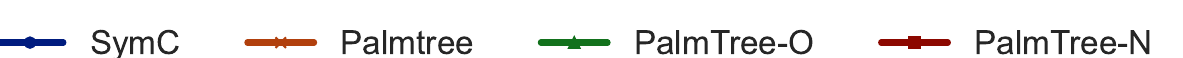}
        
        \subfloat[Reduce model size]{
        \includegraphics[width=0.49\linewidth]{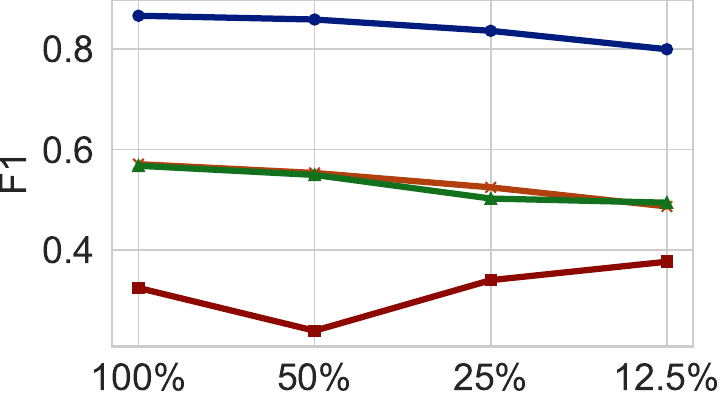}
        \label{subfig:reduce-weights}}
        \subfloat[Reduce training iterations]{
        \includegraphics[width=0.47\linewidth]{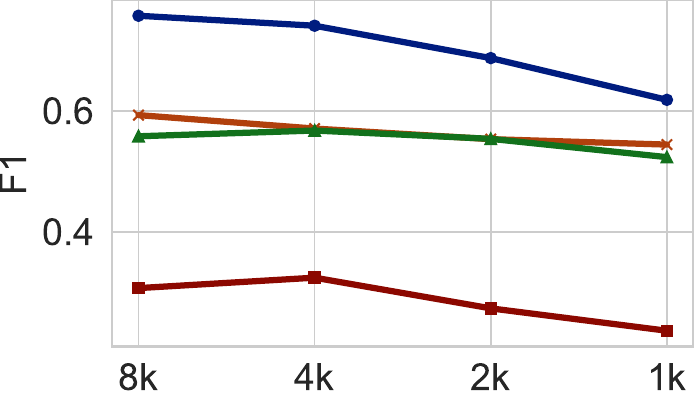}
        \label{subfig:reduce-iterations}}

        
        \caption{Comparing \sys and baselines when we (a) reduce the model weights, and (b) reduce the number of training iterations, and observe how that affects the performance.}

        
        \label{fig:resource-constraints}
\end{figure}%

    
    
    
        

\subsection{Ablations}
\label{subsec:rq4}

\para{Equivariance vs. Invariance.}
We compare the $Aut(P\!D\!G)$-\emph{equivariant} self-attention layers to the $Aut(P\!D\!G)$-\emph{invariant} ones, an alternative design choice to implement $Aut(P\!D\!G)$-\emph{invariant} code models. 
Figure~\ref{subfig:equivariant_invariant} shows that setting layers invariant early hinders prediction performance.
\sys with equivariant layers has an average 0.73 F1 across all training iterations and outperforms the second-best setting by 60.7\%.
This observation confirms the empirical findings that making earlier layers equivariant instead of invariant leads to better performance~\citep{higgins2018towards}.

\para{Adding pre-training.}
We explore the impact of pre-training \sys with masked language modeling~\citep{devlin2018bert}. 
We compare \sys (without pre-training by default) to pre-trained versions (and then fine-tuned them for memory region prediction) with varying pre-training iterations. 
Figure~\ref{subfig:pretraining} shows that pre-training with even one epoch results in a significantly improved F1 score, \eg by 10.8\%, with much faster convergence. 
However, additional pre-training epochs show diminishing returns, likely due to the limited training samples, \eg the F1 score only improves by 3.2\% with pre-training five epochs compared to 1 epoch.

\para{Non-$Aut(P\!D\!G)$-equivariant baselines.}
We aim to confirm the performance of \sys comes from preserving the symmetry, as opposed to the increased capacity from the distance matrix in the self-attention layers.
To test this hypothesis, we compare \sys with the baselines using the same exact architecture, but only changing the distance matrix whose entries are (1) kept exactly the same (\ie fully permutation equivariant baseline), and (2) derived from the relative distance between the monotonically increasing positions (\ie non-equivariant relative positional embedding).

\begin{figure*}[!t]
  \begin{center}
    \centering

    \subfloat[Equivariant vs. invariant]{
    \includegraphics[width=.28\linewidth]{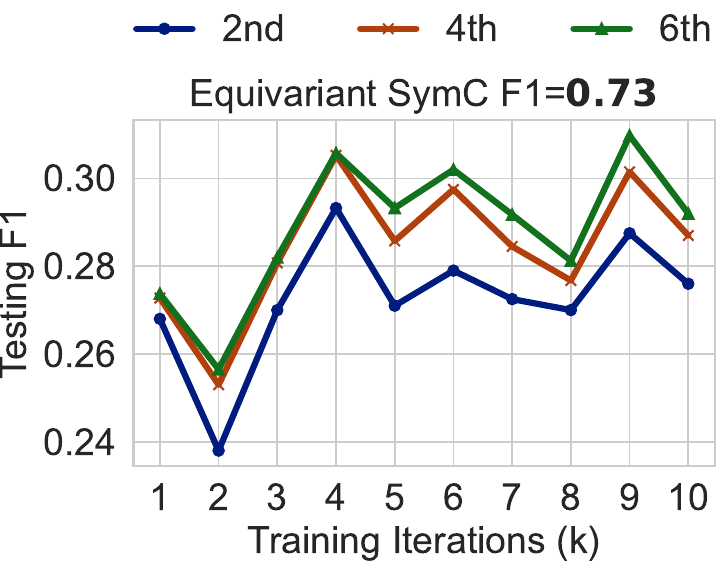}
    \label{subfig:equivariant_invariant}}
    \subfloat[Pre-training \sys]{
    \includegraphics[width=.34\linewidth]{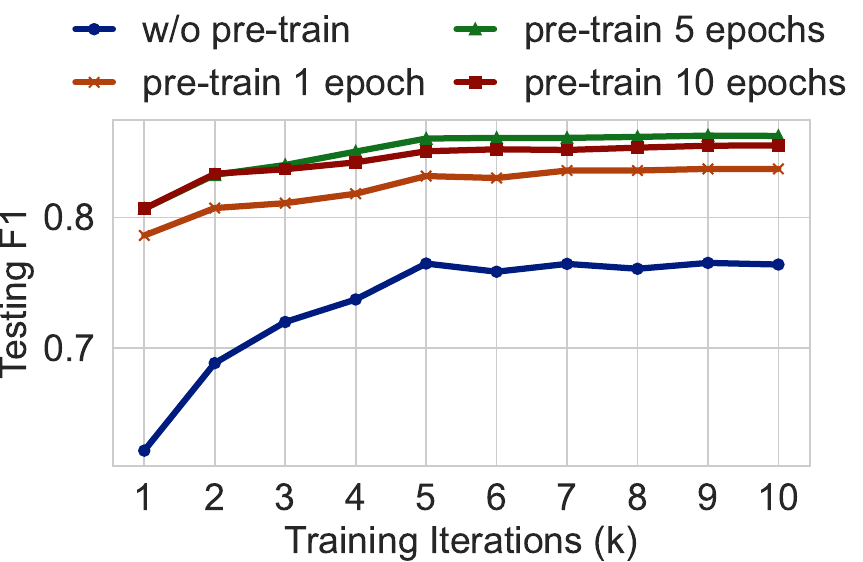}
    \label{subfig:pretraining}}
    \subfloat[Non-$Aut(P\!D\!G)$-equivariant baselines]{
    \includegraphics[width=.35\linewidth]{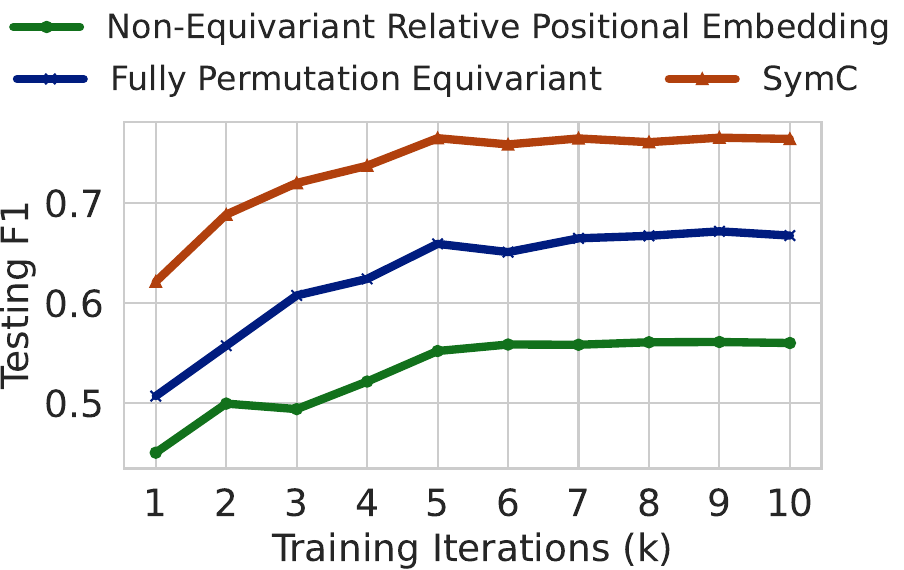}
    \label{subfig:pos}}

    \caption{(a) Setting \sys's self-attention layers \emph{invariant} in earlier layers. (b) Pre-training \sys with varying pre-training epochs. (c) Comparing the $Aut(PDG)$-equivariant layers with the fully permutation equivariant ones (distance matrix with identical values) and non-equivariant ones (default relative position embeddings~\cite{2020t5, wang2021codet5}).}
    
    
    \label{fig:ablations}
  \end{center}
\end{figure*}

Figure~\ref{subfig:pos} shows the comparison of testing F1 score by training the three models on memory region prediction tasks.
Note that both baselines adopt the same shape of the distance matrix as SymC, so their number of parameters is identical with \sys as well. 
The results demonstrate that SymC outperforms the fully permutation equivariant one and the non-equivariant one by 14.5\% and 36.4\%, respectively.

%% file: discussion.tex
\section{Limitations and Future Work}
\label{sec:discussion}

\para{Overhead.} 
\sys incurs additional overhead in generating PDG for each sample, though we have put reasonable engineering effort into making graph construction relatively cheap (Appendix~\ref{subsec:efficiency}, Figure~\ref{fig:runtime-overhead}), and we can further optimize the graph construction by interleaving it with the training/inference to further improve the efficiency. 

In essence, \sys imposes the equivariance constraints \emph{regardless of training}. 
This implies that a slight finetuning to help the model simply adapt to the symmetry constraints while specializing for downstream tasks would be likely enough~\cite{basu2023equi}. 
This is exactly why \sys can obtain strong results by training the model from scratch using the typical fine-tuning tasks (not next token prediction or masked language modeling).

\para{Other code symmetries.}
Our current framework focuses on the permutation group, but it extends to all transformations as long as they form a group. 
For example, variable name renaming is also a permutation over the entire vocabulary. 
SymC can also be readily applied to token permutation, \eg \code{a=a+1} to \code{a=1+a} is a symmetry with a proper type inference, \eg ensuring \code{a} is not a string and \code{+} is not a string concatenation operator. 

However, the expressivity of the code symmetry can be restrictive to certain transformations. 
For example, insertion and deletion operations (\eg deadcode elimination) are not invertible. 
Therefore, a new formalism, \ie semigroup~\cite{hille1996functional} which relaxes the requirement of invertibility (\S\ref{sec:preliminaries}), is required.
Moreover, it is often an expensive manual effort to identify the symmetry group structure for arbitrary code transformations. 
For example, it is unclear how a compiler optimization pass, \eg gcc O3, can be expressed using a sequence of insertion, deletion, and permutation operations. 
It would be an interesting future direction to uncover the unknown code symmetry from these transformations~\cite{huh2024discovering}.

%% file: related.tex
\section{Related Work}
\label{sec:related}

\para{Code representation learning.}
Previous research aims to automate software development tasks through code representation learning~\citep{ding2023traced, feng2020codebert, guo2022unixcoder, ahmad2021unified}.
Typical methodologies involve employing new model architectures and pre-training objectives with the goal of learning code representation that can be reused for other downstream program analysis and understanding tasks~\citep{hellendoorn2019global, bieber2020learning, pei2020trex,  pei2021stateformer, allamanis2017learning, sun2020treegen, peng2021integrating, kim2021code, guo2020graphcodebert}. 
However, unlike our approach, these approaches do not provide any guarantees that the underlying models can be robust against any semantics-preserving transformations including simple permutations. 

An alternative popular setting is to ground the program with its \emph{execution behavior} so it learns semantics-aware code representations~\cite{wen2024grounding, pei2020trex, ding2023traced, nye2021show, souza2023lexecutor, wang2020blended, bieber2022static, ye2022neural}.
One caveat of incorporating execution in \sys is that the dynamic traces are often incomplete to expose all possible code behaviors. 
Therefore, it might suffer from identifying false symmetries.
However, as conservative static analysis can miss many true symmetries, dynamic traces can assist the analysis in canonicalizing transformations that are hard to reason about statically. 
In our current formulation based on PDG, we only consider symmetries for \emph{all} input, but it can be relaxed to a subset of inputs, which is an exciting future work.

\para{Symmetry in deep learning.}
Symmetry plays a crucial role in creating efficient neural architectures in various domains~\citep{reiser2022graph, wang2020incorporating, bogatskiy2020lorentz, perraudin2019deepsphere, cohen2016group, gordon2019permutation, dehmamy2021automatic}. 
Different architectures, such as CNNs, GNNs, and Transformers, leverage symmetry as the inductive bias to build model architectures robust to geometric transformations such as translations, rotations, permutations, etc.~\citep{lee2019set, cohen2016group, esteves2018learning, hutchinson2021lietransformer, gordon2019permutation, romero2020group}. 
\sys sets the first step to formalize code semantics learning using symmetry groups. 

Recent studies have explored various strategies beyond symmetry groups to encode the geometric properties of input, aiming to enhance the model's robustness, generalization, and sample efficiency. 
These efforts include integrating symbolic operations into the neural architecture~\cite{chaudhuri2021neurosymbolic, li2023scallop, kim2017structured, hu2016harnessing} and designing specialized loss functions~\cite{fort2019deep, basu2023equi, huh2024discovering}.
Neural-symbolic approaches are particularly appealing for modeling programs, as programs typically encompass symbolic elements, such as syntactic and semantic rules, and fuzzy elements, such as natural language comments and function/variable names.
However, developing differentiable components for symbolic rules might require extensive effort from the domain expert~\cite{yi2018neural, garcez2022neural}. 
\sys represents a new alternative to express program semantics in a way that is amenable to engineering the neural architecture, thanks to the rich literature in the domain of geometric deep learning~\cite{bronstein2017geometric}.
By representing program semantics as code symmetry groups, our approach holds the promise of enhancing inference efficiency by bypassing the interpretation of symbolic rules, which often face scalability challenges in traditional program analysis.

%% file: conclusion.tex
\section{Conclusion}
\label{sec:conclusion}

We studied code symmetries' impact on code LLM architectures for program reasoning.
We introduced a novel self-attention variant that guarantees equivariance against code symmetries based on the permutation group.
We implement \sys and evaluate against various semantics-preserving transformations across different program analysis tasks, demonstrating the improved generalization in reasoning and analyzing programs.

%% file: ack.tex
\section*{Acknowledgement}

We thank the anonymous reviewers for their constructive comments and feedback, which significantly improved this paper.
This work was supported in part by NSF grants CNS-2154874, CNS-1564055, ONR grant N00014-17-1-2788, an NSF career award, a Google ASPIRE Award, multiple Google Cyber NYC awards, Columbia SEAS/EVPR Stimulus award, Columbia SEAS-KFAI  Generative AI and Public Discourse Research award, and Accenture.

%% file: impact.tex
\section*{Impact Statement}
Security-critical program analysis tools like those used for vulnerability detection and malware analysis have increasingly relied upon advanced machine learning, such as Large Language Models (LLMs), for improved efficiency, precision, and automation.
However, incorrect predictions from learned code analysis impede the deployment and introduce more vulnerabilities in the deployed systems.
This paper develops a new language to represent program semantics in a way that is amenable to being enforced in LLM architectures.
Different from the philosophy of LLMs where all the rules are expected to be learned in an entirely data-driven manner, we want to emphasize the formal aspects of our approach to provably guarantee that the LLMs follow the precisely defined rules.
Overall, this paper aims to advance generative AI technologies such as LLMs to make them trustworthy for high-assurance programming systems.

%% file: appendix.tex
\appendix

\section{Detailed Preliminaries}
\label{app_sec:preliminaries}


This section formally defines the symmetry group and the invariance and equivariance properties against the symmetry group.

\para{Symmetry group.}
Intuitively, a \emph{symmetry} is a transformation or operation on an object that preserves certain properties of the object. 
For example, in the context of image classification, a rotation operation acting on an image of a ball, which does not change the label of the ball, can be considered a symmetry. 
A symmetry group is a set of such symmetries with some additional properties. 
An arbitrary set of symmetries does not always form a symmetry group. 
To form a symmetry group, a set of operations must possess certain additional properties as described below.

\begin{definition}
\label{def:group}

A \emph{symmetry group} $(G,\circ)$ consists of a non-empty set $G$ of transformations and a binary operator $\circ:G\times G\rightarrow G$, where $\circ$ operates on two elements (\ie transformations) in $G$, \eg $x,y\in G$, and produces a new transformation $z=x\circ y$.
$(G,\circ)$ should satisfy four axioms:

\begin{itemize}[leftmargin=*]
    \item \textbf{Associativity}: $\forall x,y,z\in G, x\circ (y\circ z)=(x\circ y)\circ z$
    \item \textbf{Identity}: $\exists \mathbf{1}\in G, \forall x\in G, x\circ \mathbf{1}=\mathbf{1}\circ x$
    \item \textbf{Inverse}: $\forall x\in G, \exists x^{-1}\in G, x\circ x^{-1}=\mathbf{1}$
    \item \textbf{Closure}: $\forall x,y\in G, x\circ y\in G$
\end{itemize}

\end{definition}

\para{Action of a symmetry group.} As defined above, the elements of a $G$ are abstract transformations that become concrete when they \emph{act} on some set $X$, \ie they transform some object $x\in X$ into another object $x'\in X$ while keeping some properties of the object \emph{invariant}. Formally, an action of a symmetry group $G$ is defined as follows:

\begin{definition}
An {\bf action} $\bullet$ of a symmetry group $(G, \circ)$ is a binary operation defined on a set of objects $X$, \ie $\bullet: G\times X\rightarrow X$\footnote{In group theory literature, this is often called the left action, but we will omit ``left'' as it is the only type of action we will use in this paper.}, where

\begin{itemize}[leftmargin=*]
    \item {\bf Identity:} $\forall x\in X, \mathbf{1}\bullet x=x$
    \item {\bf Compatibility:} $\forall g,h\in G, x\in X, (g\circ h)\bullet x=g\bullet (h\bullet x)$
\end{itemize}

\end{definition}

As a concrete example, $X$ can be a set of programs and $G$ can be all possible instruction permutations that preserve the input-output behavior of the programs in $X$. It might seem unclear at this point how these permutations form a group (satisfying group axioms). 
We will formalize the notion of permutations and their actions on programs in \S\ref{subsec:finding_program_symmetry}.

\para{Notation.}
It is common in the group theory literature to use $\circ$ to denote both \emph{action} and \emph{composition}, when it is clear from the context which operation is being used~\citep{higgins2018towards}. 
For example, $(g\circ h)\circ x$ denotes \emph{composing} the two transformations $g$ and $h$ and then letting the composite transformation \emph{act} on an object $x$.
It is also customary to interchange $g(x)$ and $g\circ x$ where both denote applying a function/action on $x$.
Therefore, we treat $g\bullet (h\bullet x)$, $g\circ (h\circ x)$, and $g(h(x))$ as the same and follow this convention in the rest of this paper. 


\para{Invariance and equivariance.}
A symmetry group comes with two properties, namely \emph{invariance} and \emph{equivariance}, that formalize the concept of preservation of some properties when a set $X$ is acted upon by the symmetry group $G$. 
Invariance refers to the property that remains unchanged under the action of the symmetry group. Equivariance, on the other hand, expresses the compatibility between the action of the symmetry group and the property. 

To define this more precisely, we need to introduce a function $f: X \rightarrow Y$, where $X$ is the set under consideration and $Y$ is the co-domain representing the range of possible values associated with the property of interest. 
The function $f$ maps each element $x\in X$ to a corresponding element $y$ in the set $Y$, indicating the property's value for that particular element.
We now define the equivariance and invariance of $f$ operating on $X$ against the group operations in $G$.

\begin{definition}
Let $f: X\rightarrow Y$ be a function where $X$ and $Y$ are two sets and $G$ be the symmetry group that acts on both sets $X$ and $Y$.\footnote{We assume that $X$ and $Y$ have the same number of elements for the action of $G$ to be defined on both $X$ and $Y$.}
\begin{itemize}[leftmargin=*]
    \item \textbf{Invariant}: $f$ is called \textbf{$G$-invariant} if $\forall g\in G, \forall x\in X, f(g\circ x)=f(x)$.
    \item \textbf{Equivariant}: $f$ is called \textbf{$G$-equivariant} if $\forall g\in G, \forall x\in X, f(g\circ x)=g\circ f(x)$.
\end{itemize}

\end{definition}

Given the definition of $G$-equivariant function, we have the following lemmas:
\begin{lemma}
\label{lemma:two_g_equivariant}
    Let $f_1$, $f_2$ be two functions that are both $G$-equivariant and $h=f_1\circ f_2$ be the new function composed by $f_1$ and $f_2$. $h$ is also $G$-equivariant.
\end{lemma}


\begin{lemma}
\label{lemma:g_equivariant_and_invariant}

Let $f_1$, $f_2$ be two functions where $f_1$ is $G$-equivariant and $f_2$ is $G$-invariant, and $h=f_2\circ f_1$ be the function composed by $f_1$ and $f_2$. $h$ is $G$-invariant.
\end{lemma}


\para{Self-attention layers.}
Given the embeddings of all vertices $f_i$ from $\mathcal{IG}$, we consider a sequence of embeddings by flattening $\mathcal{IG}$ following the order of instructions in $c$. Let this sequence of embeddings be denoted as $e=(e_1, ..., e_n)$. The self-attention computation, denoted as $A$, takes $e$ as input and produces another sequence of embeddings, denoted as $(e'_1, ..., e'_n)$.

The core operations in self-attention $A$ involve updating each embedding $e_i$ through the following steps:

\begin{enumerate}[leftmargin=*, itemsep=.2cm]

\item
First, it maps each embedding $e_i$ to three embeddings (query, key, and value): $q_i=f_q(e_i)$, $k_i=f_k(e_i)$, $v_i=f_v(e_i)$, where $f_q$, $f_k$, and $f_v$ are affine transformations (\ie fully-connected linear layers) parameterized by $w_q$, $w_k$, and $w_v$, respectively.

\item
Next, it computes the attention score $a_{ij}$ between every pair of embeddings $e_i$ and $e_j$ by taking the dot product between the query $q_i$ of $e_i$ and the key $k_j$ of $e_j$: $a_{ij}=q_i\cdot k_j$. 
The attention scores form a square matrix, where each cell $a_{ij}$ indicates the attention that $e_i$ should pay to $e_j$. The attention scores are then divided by $\sqrt{d}$ (the dimension of the embedding vectors), scaled using the softmax function to ensure they sum up to 1: $\hat{a}_{ij} = \frac{\exp(a{ij})}{\sum^n_{j=1}\exp(a_{ij})}$. These two operations are denoted by $s$.

\item
Finally, the scaled attention score $\hat{a}_{ij}$ is multiplied by $v_j$, and a vector sum is computed: $e'_i = \sum^n_{j=1} \hat{a}_{ij} v_{ij}$.
\end{enumerate}

\section{Complete Proofs}
\label{app:proof}

\para{Lemma~\ref{lemma:two_g_equivariant}.}
Let $f_1$ and $f_2$ be two functions that are both $G$-equivariant and $h=f_1\circ f_2$ be the new function composed by $f_1$ and $f_2$. $h$ is also $G$-equivariant.

\begin{proof}
    For all $g\in G$ and any input $x$, we have 
    \begin{align*}
        h(g\circ x)&=(f_1 \circ f_2)(g \circ x) \\
        &= f_1(f_2(g\circ x)) & \text{\footnotesize $\triangleright$ Associativity} \\
        &= f_1(g\circ f_2(x)) & \text{\footnotesize $\triangleright$ $f_2$ is equivariant to $g$} \\
        &= g\circ f_1(f_2(x)) & \text{\footnotesize $\triangleright$ $f_1$ is equivariant to $g$} \\
        &= g\circ (f_1\circ f_2)(x) & \text{\footnotesize $\triangleright$ Associativity} \\
        &= g\circ h(x)
    \end{align*}

    Therefore, $h(g\circ x)=g\circ h(x)$, so $h$ is $G$-equivariant. 
\end{proof}

\para{Lemma~\ref{lemma:g_equivariant_and_invariant}.}
Let $f_1$ and $f_2$ be two functions where $f_1$ is $G$-equivariant $f_2$ is $G$-invariant, and $h=f_2\circ f_1$ be the new function composed by applying $f_1$ and then $f_2$. $h$ is $G$-invariant.

\begin{proof}

For all $g\in G$ and any input $x$, we have 
    \begin{align*}
        h(g\circ x)&=(f_2 \circ f_1)(g \circ x) \\
        &= f_2(f_1(g\circ x)) & \text{\footnotesize $\triangleright$ Associativity} \\
        &= f_2(g\circ f_1(x)) & \text{\footnotesize $\triangleright$ $f_1$ is equivariant to $g$} \\
        &= f_2(f_1(x)) & \text{\footnotesize $\triangleright$ $f_2$ is invariant to $g$} \\
        &= (f_2\circ f_1)(x) & \text{\footnotesize $\triangleright$ Associativity} \\
        &= h(x)
    \end{align*}
\end{proof}

\para{Theorem~\ref{theorem:automorphism}.}
The set of automorphisms $\sigma\in Aut(\mathcal{IG})$ forms a program symmetry group.

\begin{proof}
    Consider an arbitrary $\sigma\in Aut(\mathcal{IG})$. 
    Definition~\ref{def:ig_automorphism} states that for all $f_i\in\{f_1,...,f_n\}$, $\sigma(f_i)$ have the same edges as $\mathcal{IG}$ before $\sigma$ was applied. As $\sigma$ is a permutation and there is also a bijective mapping between $f_i$ and $c_i$, \ie $f_i$ always interprets $c_i$, we have $\sigma(f_i)=f_i(\sigma\circ c_i, in_i)$.
    Definition~\ref{def:ig_automorphism} also states that $\sigma(f_i)$ is connected with the same edges.
    Therefore, the output of $\sigma(f_i)=out_i$. 
    We thus have $f_i(\sigma\circ c_i, in_i)=out_i=f_i(c_i, in_i), \forall \sigma\in Aut(\mathcal{IG})$ and $\forall f_i\in\{f_1,...,f_n\}$. Therefore, all $\sigma\in Aut(\mathcal{IG})$ are semantics-preserving program symmetries, according to Definition~\ref{def:program_symmetry}.
    Moreover, it is well known in the literature that the automorphisms of any graph form a group by satisfying group axioms (Definition~\ref{def:group})~\citep{biggs1993algebraic, west2001introduction}. 
    Therefore, $Aut(\mathcal{IG})$ forms a group of program symmetries, according to Definition~\ref{def:program_symmetry-group}: $Aut(\mathcal{IG})\in G$.
    
\end{proof}

\para{Permutation matrix.} 
Let $\pi$ be a symmetry in the permutation group that permutes input embeddings $e \in \mathbb{R}^{d\times n}$ to the self-attention layer. 
Applying $\pi$ is done by $e$ with a permutation matrix $p_{\pi} \in \{0,1\}^{n\times n}$~\citep{knuth1970permutations}. 
$p_{\pi}$ is an orthogonal binary matrix with a single 1 in each column and row, and 0s elsewhere. 
Right-multiplying $e$ with $p_{\pi}$ permutes columns, and left-multiplying $e^T$ with $p_{\pi}^T$ permutes rows.

\para{Lemma~\ref{lemma:ig_invariant_token_predictor}.}
The biased self-attention layer computing the embedding $e'_i=GA(e_i)$ is $Aut(\mathcal{IG})$-invariant.

\begin{proof}
\label{proof:GA-invariant}

\begin{align*}
    e'_i&=GA(\sigma\cdot e_i) \\
    &= w_v\sigma(e)\cdot s(w_k\sigma(e)^T\cdot w_q e_i+\sigma(d_i)) \\
    \intertext{$d_i$ is a column vector, so permuting the row of $d_i$ is achieved by $p^T_{\sigma} d_i$ (see \S\ref{subsec:ig_invariant_code_analysis}):}
    &= w_v e p_{\sigma}\cdot s((w_k e p_{\sigma})^T\cdot w_q e_i+p^T_{\sigma} d_i) \\
    &= w_v e p_{\sigma}\cdot s(p_{\sigma}^T(w_k e)^T\cdot w_q e_i+p^T_{\sigma} d_i) \\
    &= w_v e(p_{\sigma}p_{\sigma}^T)\cdot s((w_k e)^T\cdot w_q e_i+ d_i) \\
    \intertext{$p_{\sigma}$ is an orthogonal matrix (see \S\ref{subsec:ig_invariant_code_analysis}):}
    &= w_v e\cdot s((w_k e)^T\cdot w_q e_i+d_i) \\ 
    &= GA(e_i) \\
\end{align*}

\end{proof}

\para{Lemma~\ref{lemma:invariant_distance_matrix}.}
The distance matrix $d$ of PDG remains invariant under the action of $\sigma\in Aut(P\!D\!G)$.

\begin{proof}
We need to show that the longest path $p_{\sigma(i)\sigma(j)}$ from $\sigma(T_{ij})$ to $\sigma(V_i)$ remains the same as $p_{ij}$ (the same applies to $n_{\sigma(i)\sigma(j)}$). Without loss of generality, we focus on proving $p_{\sigma(i)\sigma(j)}=p_{ij}$.

Assume there exists a longest path $P=(T_{ij},...,V_{i})$. Let $P'=(\sigma(T_{ij}),...,\sigma(V_{i}))$ be the corresponding longest path in $\sigma(PDG)$ under the automorphism $\sigma$. We need to demonstrate two properties.

First, $P'$ is a valid path from $\sigma(T_{ij})$ to $\sigma(V_i)$. Since $P$ is a valid path, $T_{ij}$ is adjacent to its next node in $P$ (denoted as $V_m$), and this holds for every pair of neighboring nodes until $V_i$. As $\sigma$ is an automorphism, the same adjacency relationship holds for $P'$, where $\sigma(T_{ij})$ is adjacent to $\sigma(V_m)$ and so on, until $\sigma(V_i)$. Hence, $P'$ is a valid path from $\sigma(T_{ij})$ to $\sigma(V_i)$ in PDG.

Second, we aim to show that $|P|=|P'|$, meaning $p_{\sigma(i)\sigma(j)}=p_{ij}$. Suppose, for contradiction, that $p_{\sigma(i)\sigma(j)}\neq p_{ij}$. Let's consider the case where $p_{\sigma(i)\sigma(j)}>p_{ij}$. This implies that the length of the path $P'=(\sigma(T_{ij}),\sigma(V_m),...,\sigma(V_n),\sigma(V_{i}))$ is longer than $p_{ij}$.

Now, let's apply $\sigma^{-1}$ to each node in $P'$, resulting in $\sigma^{-1}(P')$. Since $\sigma^{-1}$ is also in $Aut(PDG)$ and $\sigma^{-1}(\sigma(V))=V$ (Definition~\ref{def:group}), each pair of adjacent nodes in $P'$, after applying $\sigma^{-1}$, remains adjacent. Furthermore, the path formed by these adjacent nodes has a length of $p_{\sigma(i)\sigma(j)}$, connecting $T_{ij}$ and $V_i$ in the original PDG.

Therefore, we obtain a path in PDG connecting $T_{ij}$ and $V_i$ that is longer than $p_{ij}$, contradicting the fact that $p_{ij}$ is the longest path in PDG between $T_{ij}$ and $V_i$. Thus, we reject the assumption that $p_{\sigma(i)\sigma(j)}<p_{ij}$.

Similarly, we can prove that $p_{\sigma(i)\sigma(j)}>p_{ij}$ is also false by demonstrating its contradiction with the fact that $p_{\sigma(i)\sigma(j)}$ is the longest path in $\sigma(PDG)$.

Hence, we conclude that $p_{\sigma(i)\sigma(j)}=p_{ij}$, and as a result, the positive distance matrix $dp$ remains invariant under the action of $\sigma\in Aut(PDG)$.

By following the same steps, we can prove that $n_{\sigma(i)\sigma(j)}=n_{ij}$, demonstrating the invariance of the negative distance matrix $dn$ under the action of $\sigma\in Aut(PDG)$.

Therefore, the distance matrix $d$ remains invariant.
    
\end{proof}

\para{Lemma~\ref{lemma:comutative_distance_matrix}.}
The distance matrix $d$ of $PDG$ commutes with permutation matrix $p_{\sigma}$ of the automorphism $\sigma\in Aut(PDG)$: $d\cdot p_{\sigma}=p_{\sigma}\cdot d$.

\begin{proof}
According to Lemma~\ref{lemma:invariant_distance_matrix}, we have:

\begin{align*}
    p^T_{\sigma}\cdot d\cdot p_{\sigma}&=d \\
    p_{\sigma}\cdot p^T_{\sigma}\cdot d \cdot p_{\sigma}&=p_{\sigma}\cdot d & \text{\footnotesize $\triangleright$ Apply $p_{\sigma}$ on both side} \\
    d\cdot p_{\sigma}&=p_{\sigma}\cdot d & \text{\footnotesize $\triangleright$ $p_{\sigma}$ is orthogonal} \\
\end{align*}
\end{proof}

\begin{lemma}
Standard self-attention layer $A$ is equivariant to the group of all permutations of input sequences.
\end{lemma}

\begin{proof}

Based on the operations performed by the self-attention layer and the permutation matrix, we can show the equivariance property as follows~\citep{ji2019mathematical}:

\begin{align*}
    A&(\pi\cdot e) \\
    &= w_v\pi(e)\cdot s(w_k\pi(e)^T\cdot w_q\pi(e)) \\
    &= w_v e p_{\pi}\cdot s((w_k e p_{\pi})^T\cdot w_q e p_{\pi}) & \text{\footnotesize $\triangleright$ Apply $p_{\pi}$} \\
    &= w_v e p_{\pi}\cdot s(p_{\pi}^T(w_k e)^T\cdot w_q e p_{\pi}) \\
    &= w_v e(p_{\pi}p_{\pi}^T)\cdot s((w_k e)^T\cdot w_q e)p_{\pi} \\
    &= w_v e\cdot s((w_k e)^T\cdot w_q e)p_{\pi} & \text{\footnotesize $\triangleright$ $p_{\pi}$ is orthogonal} \\ 
    &= \pi(A(e)) \\
  \end{align*}


\end{proof}

Based on Lemma~\ref{lemma:invariant_distance_matrix} and Lemma~\ref{lemma:comutative_distance_matrix}, we now prove Theorem~\ref{theorem:ga_is_equivariant} -- the biased self-attention layer is $Aut(\mathcal{IG})$-equivariant.

\begin{proof}
\begin{align*}
    &GA(\sigma\cdot e) \\
    &= w_v\sigma(e)\cdot s(w_k\sigma(e)^T\cdot w_q\sigma(e)+\sigma(d_{\mathcal{IG}})) \\
    \intertext{$\sigma(\cdot)$ denotes applying the permutation matrix $p_{\sigma}$. As we have $\sigma(d_{\mathcal{IG}})=d_{\mathcal{IG}}$ (the first property of $d_{\mathcal{IG}}$):}
    &= w_v e p_{\sigma}\cdot s((w_k e p_{\sigma})^T\cdot w_q e p_{\sigma}+d_{\mathcal{IG}}) \\
    \intertext{Softmax $s$ is permutation equivariant, and $d_{\mathcal{IG}}\cdot p_{\sigma}=p_{\sigma}\cdot d_{\mathcal{IG}}$ (the second property of $d_{\mathcal{IG}}$):}
    &= w_v e(p_{\sigma}p_{\sigma}^T)\cdot s((w_k e)^T\cdot w_q e\cdot p_{\sigma}+d_{\mathcal{IG}}\cdot p_{\sigma}) \\
    &= w_v e\cdot s((w_k e)^T\cdot w_q e+d_{\mathcal{IG}})\cdot p_{\sigma} \\ 
    &= \sigma(GA(e)) \\
\end{align*}
\end{proof}

\section{\sys Implementation Details}
\label{app:implementation}

\para{Input sequences to self-attention.}
The Transformer self-attention layer takes an input sequence of embeddings $e$ generated by the embedding layer $Emb$. 
It consists of four input sequences: the instruction sequence $c$, per-instruction positional embeddings, and node centrality, denoted as $x_c$, $x_{pos}$, $x_{ind}$, and $x_{outd}$, respectively. 
For example, given the instruction sequence \code{a=a+1;b=a}, $x_c$ represents the tokenized sequence as $(\code{a,=,a,+,1,b,=,a})$. 
$x_{pos}$ assigns positions such that each new instruction/statement begins with position 1 of its first token and increases by 1 for each subsequent token within the instruction.

The centrality of each instruction is encoded by the in-degree and out-degree of the corresponding node in PDG. 
For each token in $c_i$, we annotate it with its in-degree (number of incoming edges) and out-degree (number of outgoing edges). For instance, in the case of \code{a=a+1;b=a}, the in-degree sequence $x_{ind}$ is $(0,0,0,0,0,1,1,1)$, and the out-degree sequence $x_{outd}$ is $(1,1,1,1,1,0,0,0)$.

We embed the four sequences independently using the embedding layers $Emb_c$, $Emb_{pos}$, $Emb_{ind}$, and $Emb_{outd}$. The final input embedding sequences $Emb(x)$ are obtained by summing the embedded sequences for each token: $Emb(x) = Emb_c(x_c) + Emb_{pos}(x_{pos}) + Emb_{ind}(x_{ind}) + Emb_{outd}(x_{outd})$.
We have the following lemma:

\begin{lemma}
\label{lemma:sum_equivariance}
    The sum of the input embedding tokens sequences is $Aut(PDG)$-equivariant: $Emb(\sigma\circ x)=\sigma\circ Emb(x)$.
\end{lemma}


Group axiom of inclusion specifies that composing the $Aut(PDG)$-equivariant embedding layers with $Aut(PDG)$-equivariant MHA layers results in an $Aut(PDG)$-equivariant representation learning component $r$ in our implementation.

\section{Detailed Experiment Setup}
\label{app:detailed_setup}

\subsection{Implementation Details}
\label{app_subsec:implementation_details}

We implement \sys using Fairseq~\citep{ott2019fairseq} PyTorch~\citep{paszke2019pytorch}. 
We conduct all the experiments on three Linux servers with Ubuntu 20.04 LTS, each featuring an AMD EPYC 7502 processor, 128 virtual cores, and 256GB RAM, with 12 Nvidia RTX 3090 GPUs in total. 

\para{PDG construction.}
To compute PDG for x86 assembly code, we utilize Ghidra to lift the assembly code into P-Code, an intermediate representation used by Ghidra, to track implicit data and control flow via the FLAGS register. 
The key advantage of using Ghidra P-Code is that it keeps the side effects of the instructions, \eg manipulating flag registers, explicit. 
For example, \code{cmp} instruction will set the zero flag implicitly at the assembly level, but Ghidra P-Code will translate it into a series of IR instructions with one of them operating on the ZF explicitly. 
We then analyze the data and control dependencies between each pair of P-Code instructions and flag the corresponding assembly code pair as dependent if at least one dependent P-Code instruction pair exists between those of the assembly code pair.

To compute PDG for Java functions, we employ JavaParser on Java ASTs for each statement to analyze control and data dependencies. 
We iterate through every pair of the statements and connect the dependent statement pairs with the directed edge. 
We extract control-flow dependencies between statements by connecting edges between different basic blocks to prevent permutations among basic blocks.

\para{Datasets.}
We use the Java dataset collected by \citet{allamanis2016convolutional} to evaluate the function name prediction. 
The dataset includes 11 Java projects, such as Hadoop, Gradle, etc., totaling 707K methods and 5.6M statements.
We fix Hadoop as our test set and use the other projects for training, to ensure the two sets do not overlap.

For binary analysis, we collect and compile 27 open-source projects, such as OpenSSL, ImageMagic, CoreUtils, SQLite, etc., which contain approximately 1.13M functions and 137.6M instructions.
We categorize the binaries based on the compilers (\texttt{GCC} or \texttt{Clang}), optimizations (\texttt{O0}-\texttt{O3}), and obfuscations (using a LLVM-based obfuscation passes based on Hikari~\cite{hikari}) and show their statistics in Table~\ref{tab:datasets}.

\begin{table}[!t]
\small
\setlength{\tabcolsep}{10pt}
\centering
\renewcommand{\arraystretch}{1.1}

\caption{The statistics of our binary dataset, categorized by compilers, optimizations, obfuscations, and lengths.}

\begin{tabular}{lrrr}
\toprule
 & \textbf{\# Files} & \textbf{\# Functions}  & \textbf{\# Instructions} \\
\midrule
\multicolumn{4}{l}{\bf Different Compilers\vspace{.1cm}}  \\
\texttt{GCC} & 1,140 & 274,840 & 33,464,420    \\
\texttt{Clang} & 1,136 & 279,832 & 31,949,673    \\ \midrule
\multicolumn{4}{l}{\bf Different Optimization Levels\vspace{.1cm}}   \\
\texttt{O0} & 285 & 98,451 & 10,202,328    \\
\texttt{O1} & 285 & 61,298 & 7,096,903   \\
\texttt{O2} & 285 & 61,298 & 7,096,903   \\
\texttt{O3} & 285 & 57,023 & 9,101,578   \\ \midrule
\multicolumn{4}{l}{\bf Different Compiler Obfuscations\vspace{.1cm}}   \\
\texttt{bcf} & 158 & 61,701 & 9,173,168    \\
\texttt{cff} & 158 & 59,724 & 11,146,990    \\
\texttt{ind} & 158 & 56,291 & 2,501,422   \\
\texttt{spl} & 158 & 61,379 & 9,652,268   \\
\texttt{sub} & 158 & 59,694 & 6,198,900   \\ \midrule
\end{tabular}

\label{tab:datasets}
\end{table}

\subsection{Experiment Configurations}
\label{subsec:experiment_configs}

\para{Program analysis tasks for evaluation.}
For source code analysis tasks, we focus on the \emph{method name prediction} and \emph{defect prediction}.
Method name prediction aims to predict the function name (in natural language) given the body of the method.
This task has been extensively evaluated by prior works to test the generalizability of code models~\cite{rabin2021generalizability}.
Following the strategies adopted in SymLM~\cite{jin2022symlm}, we tokenize the function names and formulate the function name prediction as a multi-label classification problem, \ie multiple binary classifications that predict the presence of a specific token in the vocabulary.
We then match the predicted tokens with the tokenized ground truth tokens to compute the F1 score.
We thus employ a 2-layer fully-connected network $F:\mathbb{R}^d\rightarrow \{0,1\}^L$ on top of a mean-pooled embedding from self-attention layers to ensure $Aut(PDG)$-invariance (\S\ref{subsec:ig_invariant_predictive_learning}), where $L$ is the vocabulary of all function name tokens in our dataset.

The defect prediction task is much more simplified than method name prediction. 
It is a binary classification task to predict whether a given Java method is buggy or not. 
We obtain the dataset from Defects4J~\cite{just2014defects4j}.

We consider three binary analysis tasks commonly used to evaluate ML-based approaches to security applications~\cite{li2021palmtree}.
The first task is \emph{function similarity detection}.
It aims to detect semantically similar functions, \eg those compiled by different compiler transformations (see below).
This task is often used to detect vulnerabilities, \ie by searching similar vulnerable functions in firmware, or malware analysis, \ie by searching similar malicious functions to identify the malware family~\cite{marcelli2022machine, xu2017neural}.
We leverage the pooling-based predictor (\S\ref{subsec:ig_invariant_predictive_learning}) by taking the mean of the embeddings $e$ produced by the last self-attention layer and feed that to a 2-layer fully-connected neural network $F:\mathbb{R}^d\rightarrow\mathbb{R}^d$.
We then leverage the cosine distance between the output of $F$ for a pair of function embeddings, \ie $e^1,e^2$, to compute their similarity: $cos(F(\mu(e^1)),F(\mu(e^2)))$.

The second task is \emph{function signature prediction}~\cite{chua2017neural}.
It aims to predict the number of arguments and their source-level types given the function in stripped binaries.
Similar to function similarity detection, we stack a 2-layer fully-connected network $F:\mathbb{R}^d\rightarrow L$ on top of mean-pooled embeddings from self-attention layers, which outputs the function signature label, \eg $L=\{\texttt{int},\texttt{float},...\}$.

The third task is \emph{memory region prediction}~\cite{guo2019vsa}, which aims to predict the type of memory region, \ie stack, heap, global, etc., that each memory-accessing instruction can access in a stripped binary.
As the prediction happens for each instruction, we employ the token-level predictor (\S\ref{subsec:ig_invariant_predictive_learning}) $F:\mathbb{R}^d\rightarrow L$, where $L=\{$\texttt{stack}, \texttt{heap}, \texttt{global}, \texttt{other}$\}$.

\para{Baselines.}
We consider nine baselines for function name prediction, including the dedicated models trained to predict function names~\citep{alon2019code2vec, alon2018code2seq, fernandes2018structured} and code LLMs~\cite{luo2023wizardcoder, guo2020graphcodebert, lachaux2021dobf, wang2021codet5, roziere2023code}.
For defect prediction, we excluded the models specialized for function name prediction, while including two additional code models that have been evaluated in the defect prediction task~\cite{guo2022unixcoder, feng2020codebert}.
Note that the dataset used to train the baseline LLMs might overlap with our test set.
For example, Hadoop (our test set for function name prediction, see Appendix~\ref{app:detailed_setup}) is included in BigCode~\citep{bigcode-project}, one of the widely used datasets to train code LLMs.
Our goal is to demonstrate \sys still generalizes better than existing code LLMs under such a disadvantaged setting.

For tasks (3)-(5), we compare to PalmTree~\citep{li2021palmtree}, the only binary code model that has evaluated on all our considered tasks. 
To ensure a fair comparison, we include three PalmTree versions: \pto, \pts, and \ptu. \pto is pre-trained on \emph{2.25 billion} instructions. \pts is pre-trained on \emph{137.6 million} instructions using our own dataset (Appendix~\ref{app:detailed_setup}), with full access to fine-tuning and evaluation data (excluding labels), while not accessible by \sys as it is not pre-trained. 
We aim to show \sys's strong generalizability even in this disadvantaged setting. 
\ptu serves as the baseline Transformer encoder without being pre-trained.

\para{Transformations.}
We consider a set of semantics-preserving transformations beyond PDG automorphisms to evaluate how preserving $Aut(PDG)$-equivariant improves \sys's generalizability.
Notably, some of these program transformations (described below) have enabled instruction reordering, which inherently performs instruction permutations.

We consider two categories of binary code transformations: (1) \emph{compiler optimizations}, where we examine 4 optimization levels (\texttt{O0-O3}) from GCC-7.5 and Clang-8, some of which involve instruction permutations, like reordering for scheduling purposes (\texttt{-fdelayed-branch}, \texttt{-fschedule-insns}); and (2) \emph{compiler-based obfuscations}, where we follow SymLM~\cite{jin2022symlm} by using 5 obfuscations written in LLVM, \ie control flow flattening (\texttt{cff}), instruction substitution (\texttt{sub}), indirect branching (\texttt{ind}), basic block split (\texttt{spl}), and bogus control flow (\texttt{bcf}), which inherently include reordering instructions, \eg adding a trampoline.

\para{Hyperparameters.}
We use \sys with 8 attention layers, 12 attention heads, and a maximum input length of 512. 
For training, we use 10 epochs, a batch size of 64, and 14K/6K training/testing samples (strictly non-overlapping) unless stated otherwise. 
We employ 16-bit weight parameters for \sys to optimize for memory efficiency.

\para{Evaluation metrics.}
For most analysis tasks (\S\ref{sec:experimental_setup}), we use \emph{F1 score}, the harmonic mean of precision and recall. 
We follow the existing works~\cite{jin2022symlm} and adopt their definition of F1 beyond the binary classifier. 
Take function name prediction as an example, we first tokenize both the ground truth and the predicted function names into a set of tokens, \ie $W$ and $W'$, respectively. 
In this case, precision measures out of $W'$, how many tokens in $W'$ appear in $W$: $precision=\frac{W\cap W'}{|W'|}$, and recall measures out of all the tokens in $W$, how many of them are correctly predicted in $W'$: $recall=\frac{W\cap W'}{|W|}$. 
We measure the precision and recall and compute the F1 score for each sample accordingly. We then average them across all samples.

For function similarity detection, as the cosine distance between two function embeddings can be an arbitrary real value between -1 and 1, a threshold is needed to determine whether pairs are similar or not. 
Therefore, we employ the ROC curve by varying the thresholds and measuring the corresponding True Positive Rate (TPR): 
$\text{TPR} = \frac{\text{True Positives}}{\text{True Positives} + \text{False Negatives}}$ and False Positive Rate (FPR): $\text{FPR} = \frac{\text{False Positives}}{\text{False Positives} + \text{True Negatives}}$. 
The ROC curve is then plotted with FPR at the x-axis and TPR at the y-axis. 
Following~\citet{li2021palmtree}, we leverage the Area Under Curve (AUC) score of the ROC curve to quantify the performance for ease of comparison. 

We note that AUC-ROC might not be the most reliable metric~\citep{arp2022and}, but we choose it primarily for comparing to the baselines whose results are measured in AUC-ROC~\citep{li2021palmtree}.

\section{Additional Experiments}
\label{app:additional_eval}

\begin{table*}[!t]

\setlength{\tabcolsep}{3pt}
\renewcommand{\arraystretch}{1.1}

\caption{Complete evaluation statistics on samples under different percentages of semantics-preserving permutations. F1 measures the prediction performance of function name, function signature, and memory region. AUC (area under the ROC curve) measures the function similarity detection performance. The violation rate is highlighted in \colorbox{lightred-medium}{red}. The larger the violation rate, the darker the color. }

\label{tab:app-unseen-permute}

\begin{center}
\begin{tabular}{llll|lllll|llll}
\toprule[1.1pt]
 & & \multirow{2}{*}{\begin{tabular}[c]{@{}l@{}}\bf Model\\\bf Size\end{tabular}} & \multirow{2}{*}{\begin{tabular}[c|]{@{}l@{}}\bf Train\\\bf Size\end{tabular}} & \multicolumn{5}{c|}{\bf F1 \& AUC} & \multicolumn{4}{c}{\bf Invariance Violation (\%)} \\ 
 & & & & Before & $w=1$ & $w=2$ & $w=3$ & $w=4$ & $w=1$ & $w=2$ & $w=3$ & $w=4$  \\ \midrule[.9pt]
\multirow{10}{*}{\begin{tabular}[c]{@{}l@{}}Function\\ Name\end{tabular}} & \sys & 68.4M & 202M & \cellcolor{shadecolor}0.363 &  0.364$^*$  & 0.363  & 0.363 & 0.363 & 0 & 0.1$^*$ & 0 & 0 \\
 & code2seq &  6.3M   & 5.1G & \cellcolor{shadecolor}0.255 & 0.238 & 0.236 & 0.237 & 0.247  & \cellcolor{lightred-darkest}54 & \cellcolor{lightred-darkest}53 & \cellcolor{lightred-darkest}57 & \cellcolor{lightred-darkest}61 \\
 & code2vec &   348M  & 32G & \cellcolor{shadecolor}0.177 & 0.199 & 0.195 & 0.197 & 0.196 & \cellcolor{lightred-darkest}53 & \cellcolor{lightred-darkest}53 & \cellcolor{lightred-darkest}52 & \cellcolor{lightred-darkest}52 \\
 & CodeLlama & 7B & N/A & \cellcolor{shadecolor}0.317 & 0.317 & 0.314 & 0.303 & 0.314  & \cellcolor{lightred-lightest}19 & \cellcolor{lightred-lightest}18 & \cellcolor{lightred-medium}19 & \cellcolor{lightred-medium}18 \\
 & CodeT5    & 770M & N/A & \cellcolor{shadecolor}0.254 & 0.254 & 0.254 & 0.254 & 0.254  & \cellcolor{lightred-lightest}9 & \cellcolor{lightred-lightest}13 & \cellcolor{lightred-lightest}15 & \cellcolor{lightred-lightest}16\\
 & DOBF    & 428M & N/A & \cellcolor{shadecolor}0.163 & 0.182 & 0.182 & 0.175 & 0.201  & \cellcolor{lightred-darkest}22 & \cellcolor{lightred-darkest}28 & \cellcolor{lightred-darkest}36 & \cellcolor{lightred-darkest}41\\
 & GGNN     &  53M   & 2.4G & \cellcolor{shadecolor}0.016 & 0.016 & 0.016 & 0.016 & 0.016 & \cellcolor{lightred-lightest}4 & \cellcolor{lightred-lightest}4 & \cellcolor{lightred-lightest}5 & \cellcolor{lightred-lightest}7 \\
 & GPT-4    & N/A & N/A & \cellcolor{shadecolor}0.303 & 0.313 & 0.317 & 0.329 & 0.307  & \cellcolor{lightred-darkest}42 & \cellcolor{lightred-darkest}43 & \cellcolor{lightred-darkest}45 & \cellcolor{lightred-darkest}43\\
 & GraphCodeBERT    & 481M & N/A & \cellcolor{shadecolor}0.208 & 0.205 & 0.212 & 0.202 & 0.206  & \cellcolor{lightred-medium}13 & \cellcolor{lightred-darkest}22 & \cellcolor{lightred-darkest}28 & \cellcolor{lightred-darkest}31\\
 & WizardCoder & 3B & N/A & \cellcolor{shadecolor}0.339 & 0.347 & 0.348 & 0.359 & 0.346  & \cellcolor{lightred-lightest}6 & \cellcolor{lightred-lightest}7 & \cellcolor{lightred-medium}12 & \cellcolor{lightred-medium}14 \\
 \midrule[.9pt]
\multirow{6}{*}{\begin{tabular}[c]{@{}l@{}}Defect\\ Prediction\end{tabular}} & \sys & 67.7M & 720K & \cellcolor{shadecolor}0.688 & - & - & - & 0.688 & - & - & - & 0 \\
 & CodeBERT & 476M & N/A & \cellcolor{shadecolor}0.622 & - & - & - & 0.617 & \cellcolor{lightred-medium}- & \cellcolor{lightred-medium}- & \cellcolor{lightred-medium}- & \cellcolor{lightred-medium}4.1 \\
 & CodeT5 & 770M & N/A & \cellcolor{shadecolor}0.633 & - & - & - & 0.6  & \cellcolor{lightred-darkest}- & \cellcolor{lightred-darkest}- & \cellcolor{lightred-darkest}- & \cellcolor{lightred-darkest}6 \\
 & DOBF & 428M & N/A & \cellcolor{shadecolor}0.624 & - & - & - & 0.615  & \cellcolor{lightred-darkest}- & \cellcolor{lightred-darkest}- & \cellcolor{lightred-darkest}- & \cellcolor{lightred-darkest}2.7 \\
 & GraphCodeBERT & 481M & N/A & \cellcolor{shadecolor}0.617 & - & - & - & 0.617  & \cellcolor{lightred-lightest}- & \cellcolor{lightred-lightest}- & \cellcolor{lightred-lightest}- & \cellcolor{lightred-lightest}1.3 \\
 & UnixCoder & 504M & N/A & \cellcolor{shadecolor}0.671 & - & - & - & 0.671  & \cellcolor{lightred-medium}- & \cellcolor{lightred-medium}- & \cellcolor{lightred-medium}- & \cellcolor{lightred-medium}2.9 \\ \midrule[.9pt]
\multirow{4}{*}{\begin{tabular}[c]{@{}l@{}}Function\\ Signature\end{tabular}} & \sys & 58.3M & 12M & \cellcolor{shadecolor}0.88 & 0.88 & 0.88 & 0.88 & 0.88 & 0 & 0 & 0 & 0 \\
 & \pto & 3.2M & 17.4G & \cellcolor{shadecolor}0.59 & 0.55 & 0.49 & 0.42 & 0.41  & \cellcolor{lightred-medium}12 & \cellcolor{lightred-medium}23 & \cellcolor{lightred-medium}18 & \cellcolor{lightred-medium}24 \\
 & \pts & 3.2M & 5.3G & \cellcolor{shadecolor}0.49 & 0.48 & 0.45 & 0.41 & 0.41  & \cellcolor{lightred-medium}19 & \cellcolor{lightred-lightest}6 & \cellcolor{lightred-medium}12 & \cellcolor{lightred-lightest}6 \\
 & \ptu & 3.2M & 614M & \cellcolor{shadecolor}0.19 & 0.41 & 0.41 & 0.41 & 0.41  & \cellcolor{lightred-darkest}83 & \cellcolor{lightred-darkest}82 & \cellcolor{lightred-darkest}83 & \cellcolor{lightred-darkest}86 \\ \midrule[.9pt]
\multirow{4}{*}{\begin{tabular}[c]{@{}l@{}}Memory\\ Region\end{tabular}} & \sys & 58.9M & 340M & \cellcolor{shadecolor}0.86 & 0.86 & 0.86 & 0.86 & 0.86 & 0 & 0 & 0 & 0 \\
 & \pto & 3.07M & 17.9G & \cellcolor{shadecolor}0.57 & 0.45 & 0.45 & 0.48 & 0.43  & \cellcolor{lightred-medium}17 & \cellcolor{lightred-medium}17 & \cellcolor{lightred-medium}28 & \cellcolor{lightred-medium}18 \\
 & \pts & 3.07M & 5.8G & \cellcolor{shadecolor}0.57 & 0.42 & 0.45 & 0.47 & 0.44  & \cellcolor{lightred-medium}10 & \cellcolor{lightred-medium}13 & \cellcolor{lightred-medium}14 & \cellcolor{lightred-medium}11 \\
 & \ptu & 3.07M & 1.1G & \cellcolor{shadecolor}0.32 & 0.22 & 0.29 & 0.17 & 0.2   & \cellcolor{lightred-medium}30 & \cellcolor{lightred-medium}36 & \cellcolor{lightred-medium}31 & \cellcolor{lightred-medium}32 \\ \midrule[.9pt]
\multirow{4}{*}{\begin{tabular}[c]{@{}l@{}}Function\\ Similarity\end{tabular}} & \sys & 58.9M & 133M & \cellcolor{shadecolor}0.96 & 0.96 & 0.96 & 0.96 & 0.96 & 0 & 0 & 0 & 0 \\
 & \pto & 3.06M & 17.4G & \cellcolor{shadecolor}0.72 & 0.61 & 0.53 & 0.71 & 0.69  & \cellcolor{lightred-medium}18 & \cellcolor{lightred-medium}19 & \cellcolor{lightred-medium}30 & \cellcolor{lightred-medium}31 \\
 & \pts & 3.06M & 5.3G & \cellcolor{shadecolor}0.8 & 0.79 & 0.76 & 0.72 & 0.72  & \cellcolor{lightred-medium}30 & \cellcolor{lightred-medium}28 & \cellcolor{lightred-medium}30 & \cellcolor{lightred-medium}35 \\
 & \ptu & 3.06M & 614M & \cellcolor{shadecolor}0.71 & 0.64 & 0.56 & 0.66 & 0.72   & \cellcolor{lightred-medium}11 & \cellcolor{lightred-medium}18 & \cellcolor{lightred-medium}24 & \cellcolor{lightred-medium}38 \\ \bottomrule[1.1pt]
\multicolumn{13}{l}{\begin{tabular}[c]{@{}l@{}}\scriptsize $^*$We observe a slight value change due to the floating point precision error by adopting memory-efficient 16-bit.\end{tabular}}
 
\end{tabular}
\end{center}
\end{table*}

\subsection{Generalization and Robustness}

Table~\ref{tab:app-unseen-permute} shows the complete results of \sys and baselines against semantics-preserving code transformations across different analysis tasks.
$w$ measures the number of steps of applying (non-repeated) 2-statement permutations by permuting different statements. 
The larger the $w$, the more permutations are applied.
While we do not observe a clear trend that a higher violation rate or decreased performance correlates with the number of times applying permutation, \sys consistently shows consistent performance and 0 violation rate.
On the contrary, all the baselines are not robust against the permutations and have their labels changed, \ie by up to 86\% violation rate.

\begin{table*}[!t]

\centering
\setlength{\tabcolsep}{6pt}
\renewcommand{\arraystretch}{1}

\caption{The performance (F1) of \sys and baselines against different unseen code transformations.}

\label{tab:app-unseen-source}

\begin{tabular}{lrrrrrrrr}
\toprule[1.1pt]
         \textbf{Transform}   & \textbf{Applied} & \textbf{\sys} & \textbf{code2seq} & \textbf{code2vec} & \textbf{CodeLlama} & \textbf{GGNN} & \textbf{GPT-4} & \textbf{WizardCoder} \\ \midrule[.9pt] 
\multirow{2}{*}{\begin{tabular}[c]{@{}l@{}}Variable\\ Rename\end{tabular}}        &Before & \cellcolor{shadecolor}\textbf{0.389}  & \cellcolor{shadecolor}0.334  & \cellcolor{shadecolor}0.264  & \cellcolor{shadecolor}0.461 & \cellcolor{shadecolor}0.029  & \cellcolor{shadecolor}0.356 & \cellcolor{shadecolor} 0.362\\
& After & \textbf{0.375}  & 0.335  & 0.247  & 0.426 &  0.026   & 0.351   & 0.361  \\ \midrule[.9pt]
\multirow{2}{*}{\begin{tabular}[c]{@{}l@{}}Statement\\ Permute\end{tabular}}        &Before & \cellcolor{shadecolor}\textbf{0.363}  & \cellcolor{shadecolor}0.241  & \cellcolor{shadecolor}0.177  & \cellcolor{shadecolor}0.317 & \cellcolor{shadecolor}0.019  & \cellcolor{shadecolor} 0.303 & \cellcolor{shadecolor} 0.339 \\
& After & \textbf{0.363} & 0.234  & 0.196  & 0.314 &  0.019   &  0.307  &  0.346 \\ \midrule[.9pt]
\multirow{2}{*}{\begin{tabular}[c]{@{}l@{}}Loop\\ Exchange\end{tabular}}        &Before & \cellcolor{shadecolor}\textbf{0.373} & \cellcolor{shadecolor}0.283  & \cellcolor{shadecolor}0.243  & \cellcolor{shadecolor} 0.414& \cellcolor{shadecolor}0.007  & \cellcolor{shadecolor} 0.310 & \cellcolor{shadecolor} 0.379 \\
& After & \textbf{0.357}  & 0.299  & 0.241  & 0.399 & 0.007   &  0.308  & 0.366  \\ \midrule[.9pt]
\multirow{2}{*}{\begin{tabular}[c]{@{}l@{}}Boolean\\ Exchange\end{tabular}}        &Before & \cellcolor{shadecolor}\textbf{0.421}  & \cellcolor{shadecolor}0.332  & \cellcolor{shadecolor}0.268  & \cellcolor{shadecolor} 0.360 & \cellcolor{shadecolor}0.031  & \cellcolor{shadecolor} 0.329 & \cellcolor{shadecolor} 0.414\\
& After & \textbf{0.412}  & 0.272  & 0.242  & 0.447 & 0.026   & 0.323  & 0.406  \\ \midrule[.9pt]
\multirow{2}{*}{\begin{tabular}[c]{@{}l@{}}Unused\\ Statement\end{tabular}}        &Before & \cellcolor{shadecolor}\textbf{0.347}  & \cellcolor{shadecolor}0.296  & \cellcolor{shadecolor}0.267  & \cellcolor{shadecolor}0.429 & \cellcolor{shadecolor}0.016  & \cellcolor{shadecolor} 0.316 & \cellcolor{shadecolor} 0.358\\
& After & \textbf{0.342}  & 0.285  & 0.26  & 0.428 &  0.012   & 0.309   & 0.350 \\ \midrule[.9pt]
\multirow{2}{*}{\begin{tabular}[c]{@{}l@{}}Switch\\ to If\end{tabular}}        &Before & \cellcolor{shadecolor}\textbf{0.372}  & \cellcolor{shadecolor}0.31  & \cellcolor{shadecolor}0.376  & \cellcolor{shadecolor}0.430 & \cellcolor{shadecolor}0.027  & \cellcolor{shadecolor} 0.326 & \cellcolor{shadecolor} 0.385 \\
& After & \textbf{0.372}  & 0.293  & 0.33  & 0.429 & 0.009   & 0.332   & 0.379 \\ \bottomrule[1.1pt]

\end{tabular}
\end{table*}

Table~\ref{tab:app-unseen-source} shows the complete results when evaluating \sys and other baselines on new samples transformed by the semantics-preserving transformations that have never been presented in the training.
We integrate CodeWordNet~\citep{jin2022symlm} to relax predicted names to a cluster of synonyms, addressing the issue of ambiguity of function names.
However, the performance of \sys decreases to 0.309 (was 0.374) when we measure the exact match.
As discussed in \S\ref{subsec:rq1-rq2}, we observe that \sys outperforms the strong baselines, \eg code2seq, by 30.8\%. 

\para{Unseen optimizations.}
We vary the compiler optimizations in training and evaluation and include reference experiments where the training and evaluation share the same optimization options (marked in gray). 
For function similarity detection, training on \texttt{O0}-\texttt{O1} means the function pair has one function compiled with \texttt{O0} and the other with \texttt{O1}.
In the case of evaluating on unseen optimizations, the corresponding testing set has to come from those compiled with \texttt{O2}-\texttt{O3} to ensure the optimizations are unseen.

Figure~\ref{fig:unseen-opt} shows that \sys outperforms \pto by 31\% when evaluated on unseen optimizations. 
\sys experiences a performance drop (\eg by 28.6\%) when not trained on \texttt{O0} but tested on those compiled with \texttt{O0}. 
We believe this drop is caused by the extensive optimizations already enabled at the \texttt{O1} (\eg GCC employs 47 optimizations to aggressively reduce execution time and code size). 
The shift in distribution between \texttt{O1} and \texttt{O0} is much more pronounced than between \texttt{O2} and \texttt{O1}, indicated by a KL divergence of 1.56 from \texttt{O1} to \texttt{O0} compared to 0.06 (96.2\% lower) from \texttt{O3} to \texttt{O2}.
Nevertheless, when evaluated on seen optimizations, \sys outperforms \pto by 28.1\% on average.

\begin{figure*}[!t]
\centering

\includegraphics[width=.6\linewidth]{./figs/rq1/cross-opt/cross-opt-legend.pdf}

\subfloat[Train \texttt{O0}-\texttt{O1}]{
\includegraphics[width=0.18\linewidth]{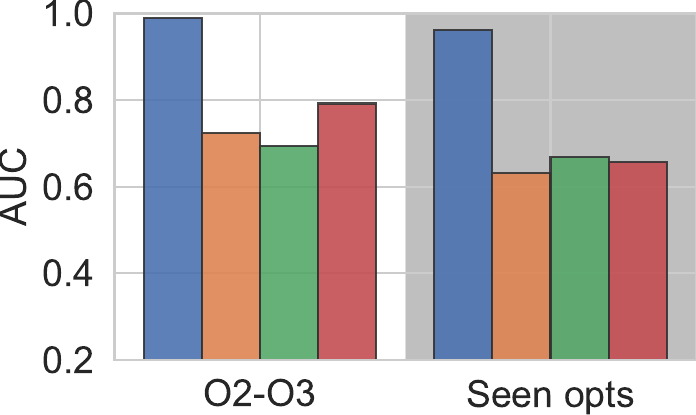}
\label{subfig:similarity-O0-O1}}\hspace{-.2em}%
\subfloat[Train \texttt{O0}-\texttt{O2}]{
\includegraphics[width=0.15\linewidth]{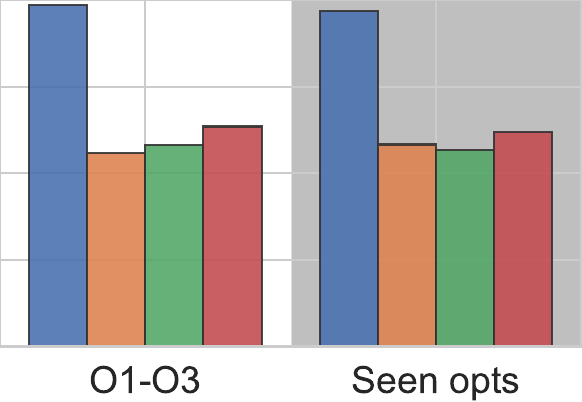}
\label{subfig:similarity-O0-O2}}\hspace{-.2em}%
\subfloat[Train \texttt{O0}-\texttt{O3}]{
\includegraphics[width=0.15\linewidth]{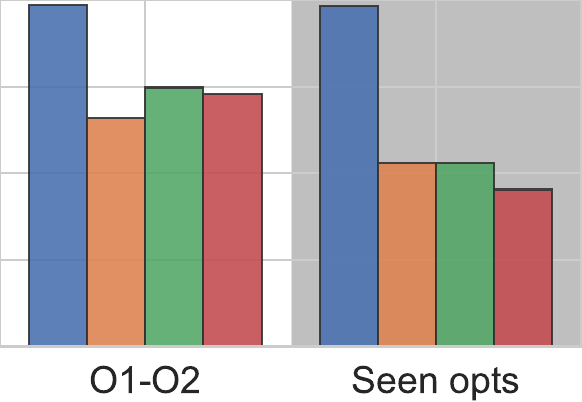}
\label{subfig:similarity-O0-O3}}\hspace{-.2em}%
\subfloat[Train \texttt{O1}-\texttt{O2}]{
\includegraphics[width=0.15\linewidth]{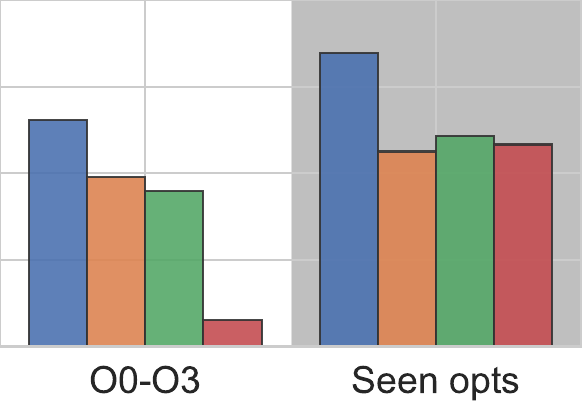}
\label{subfig:similarity-O1-O2}}\hspace{-.2em}%
\subfloat[Train \texttt{O1}-\texttt{O3}]{
\includegraphics[width=0.15\linewidth]{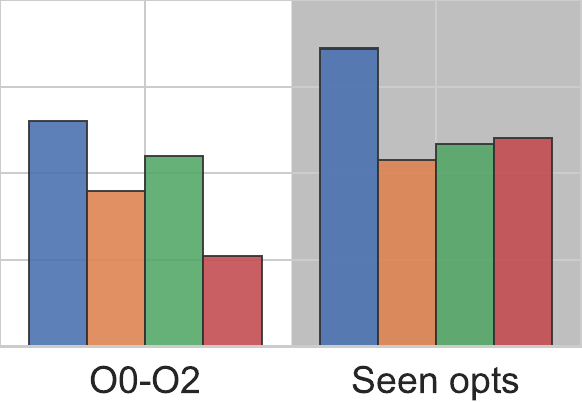}
\label{subfig:similarity-O1-O3}}\hspace{-.2em}%
\subfloat[Train \texttt{O2}-\texttt{O3}]{
\includegraphics[width=0.15\linewidth]{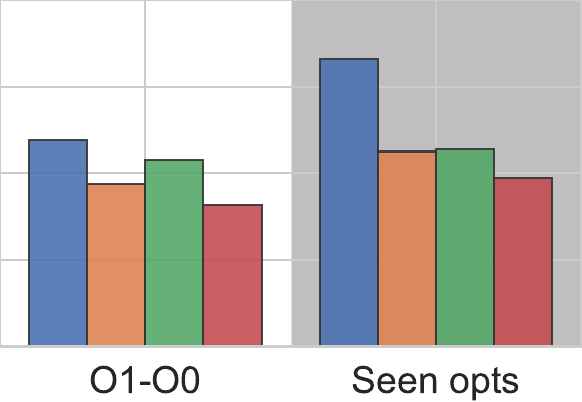}
\label{subfig:similarity-O2-O3}}

\subfloat[Train \texttt{O0}]{
\includegraphics[width=0.23\linewidth]{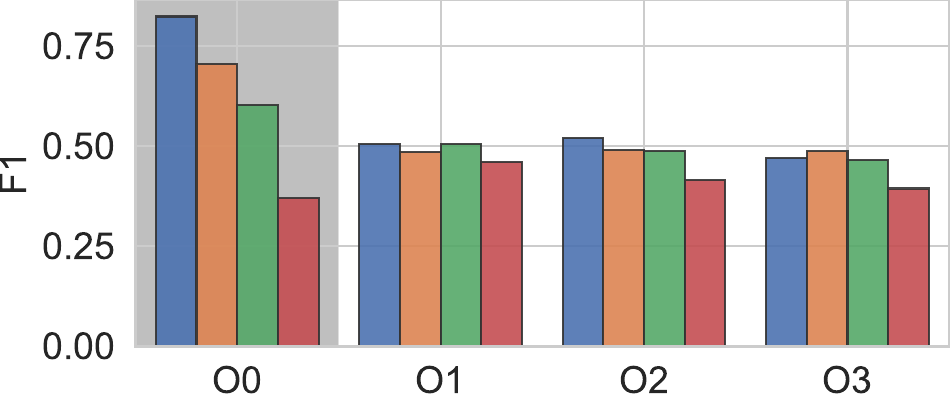}
\label{subfig:signature-O0}}
\subfloat[Train \texttt{O1}]{
\includegraphics[width=0.23\linewidth]{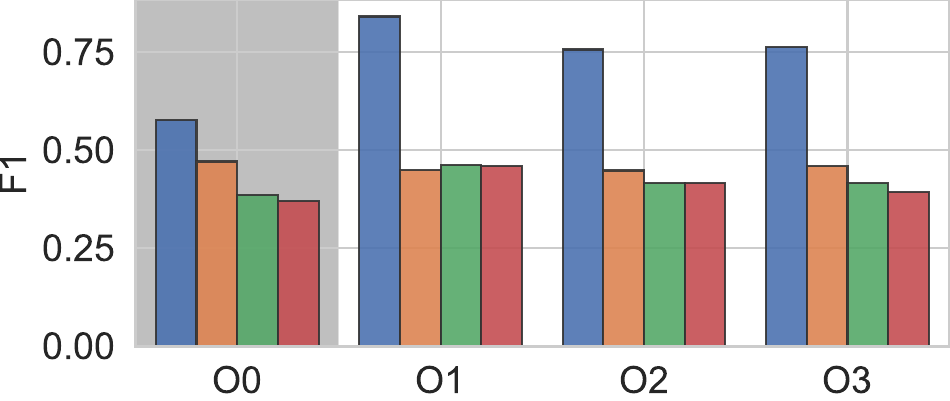}
\label{subfig:signature-O1}}
\subfloat[Train \texttt{O2}]{
\includegraphics[width=0.23\linewidth]{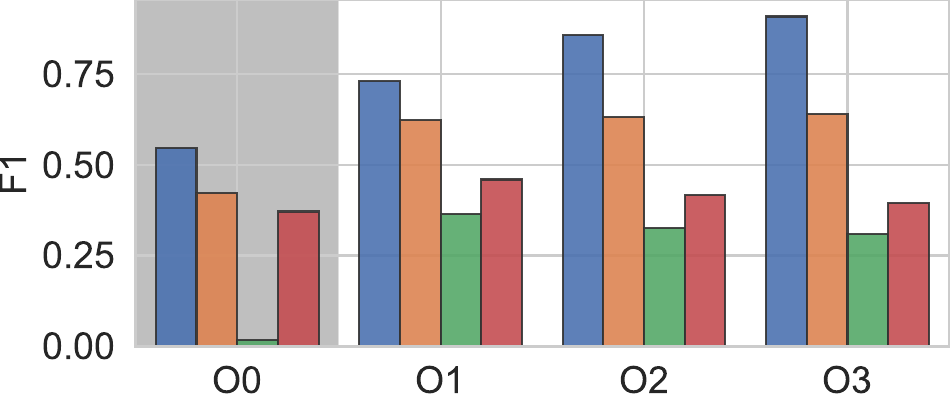}
\label{subfig:signature-O2}}
\subfloat[Train \texttt{O3}]{
\includegraphics[width=0.23\linewidth]{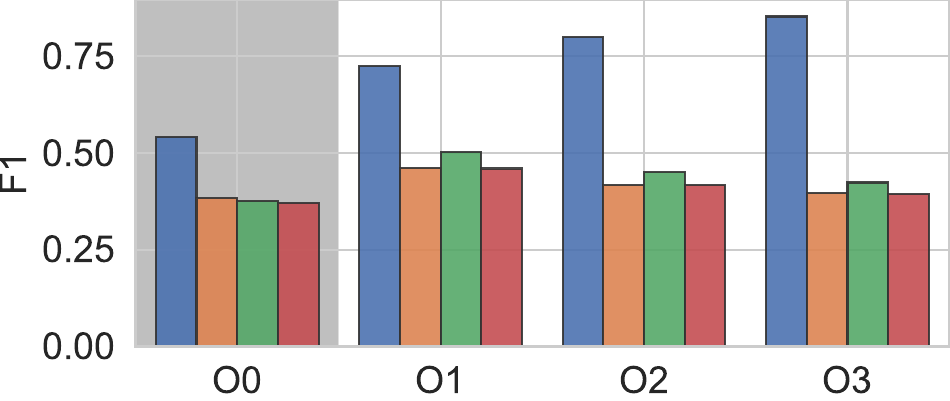}
\label{subfig:signature-O3}}

\caption{Unseen optimization evaluation. The upper row, \ie (a)-(f), shows the results on function similarity detection. The lower row, \ie (g)-(j), are results on function signature prediction. We also include the evaluation on seen optimizations (marked in gray).}
\label{fig:unseen-opt}
\end{figure*}


\para{Unseen obfuscations.} We compare \sys to baselines on generalization to unseen obfuscations. 
Figure~\ref{fig:unseen-obf} shows that \sys outperforms \pto (on average) on unseen and seen obfuscations by 33.3\% and 36.6\%, respectively. 
Similar to the observations in evaluating unseen optimizations, while the obfuscations are not directly related to instruction permutations (\ie automorphisms in $Aut(P\!D\!G)$), \sys maintains its superior performance.

\begin{figure*}[!t]
\centering

\includegraphics[width=.6\linewidth]{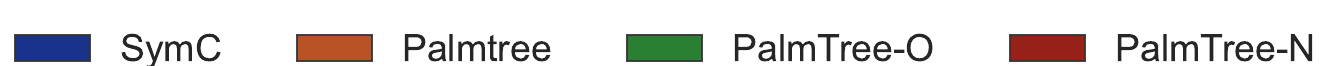}

\subfloat[Train \texttt{sub}]{
\includegraphics[width=0.22\linewidth]{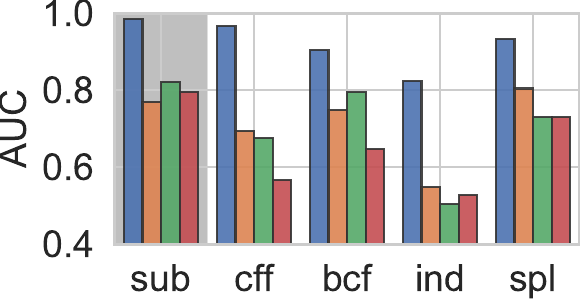}
\label{subfig:similarity-sub}}\hspace{-.2em}%
\subfloat[Train \texttt{cff}]{
\includegraphics[width=0.18\linewidth]{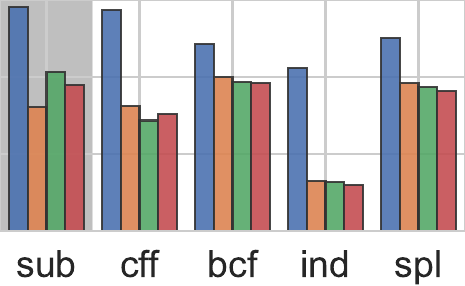}
\label{subfig:similarity-cff}}\hspace{-.2em}%
\subfloat[Train \texttt{bcf}]{
\includegraphics[width=0.18\linewidth]{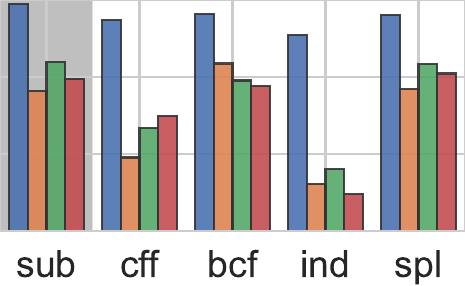}
\label{subfig:similarity-bcf}}\hspace{-.2em}%
\subfloat[Train \texttt{ind}]{
\includegraphics[width=0.18\linewidth]{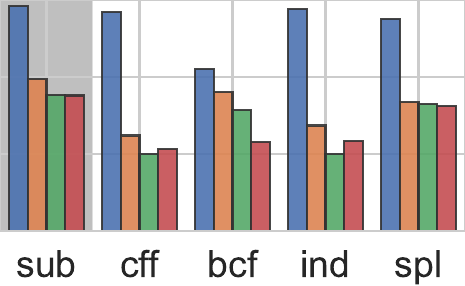}
\label{subfig:similarity-ind}}\hspace{-.2em}%
\subfloat[Train \texttt{spl}]{
\includegraphics[width=0.18\linewidth]{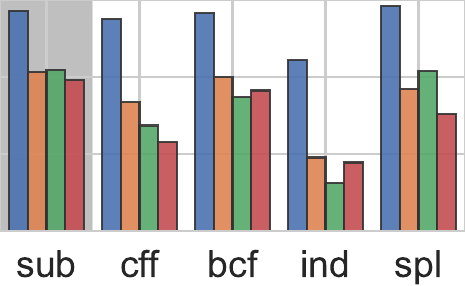}
\label{subfig:similarity-spl}}

\subfloat[Train \texttt{sub}]{
\includegraphics[width=0.22\linewidth]{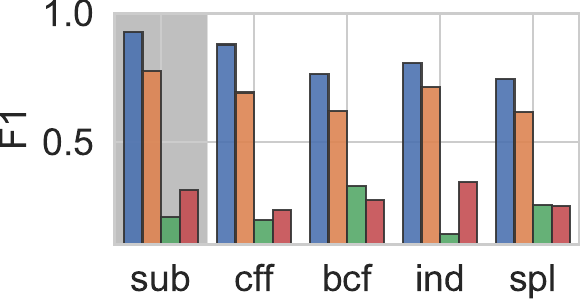}
\label{subfig:signature-sub}}\hspace{-.2em}%
\subfloat[Train \texttt{cff}]{
\includegraphics[width=0.18\linewidth]{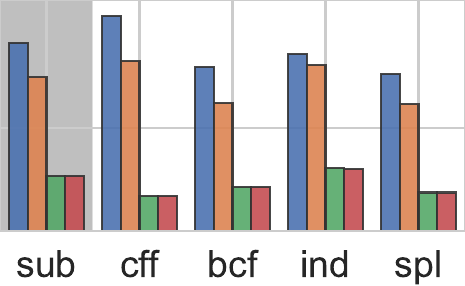}
\label{subfig:signature-cff}}\hspace{-.2em}%
\subfloat[Train \texttt{bcf}]{
\includegraphics[width=0.18\linewidth]{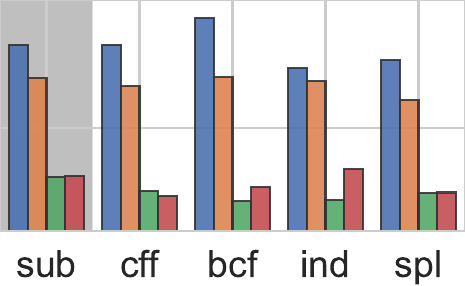}
\label{subfig:signature-bcf}}\hspace{-.2em}%
\subfloat[Train \texttt{ind}]{
\includegraphics[width=0.18\linewidth]{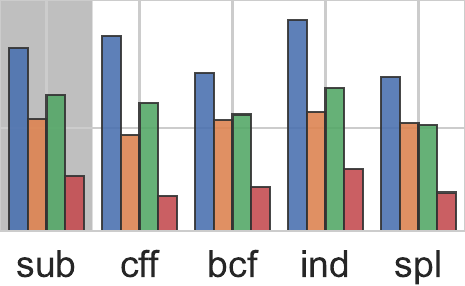}
\label{subfig:signature-ind}}\hspace{-.2em}%
\subfloat[Train \texttt{spl}]{
\includegraphics[width=0.18\linewidth]{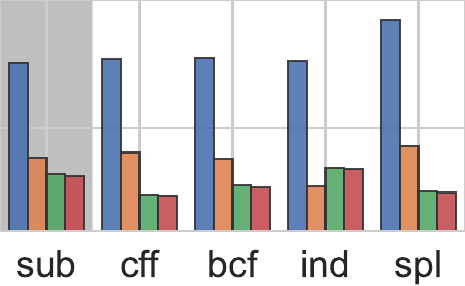}
\label{subfig:signature-spl}}

\caption{Unseen obfuscations evaluation. Similar to Figure~\ref{fig:unseen-opt}, the upper row, \ie (a)-(e), shows the results on function similarity detection. The lower row, \ie (f)-(j), are results on function signature prediction. We also include the evaluation on seen optimizations (marked in gray).}
\label{fig:unseen-obf}
\end{figure*}

\para{Unseen lengths.}
Besides the code transformations, we look into \sys's generalization to \emph{longer} sequences than those seen in training, a popular task for evaluating model generalizability~\citep{gordon2019permutation}. We divide samples into four length bins (\texttt{bin1} to \texttt{bin4}) based on their distribution in the dataset (\S\ref{sec:experimental_setup}). 
The bins are non-overlapping and increase in length. For example, we used bins [0-10], [1-20], [21-50], and [51-500] for function similarity detection. 
Figure~\ref{fig:unseen-length} demonstrates that \sys maintains strong generalization to longer sequences, outperforming \pto by 41.8\%.

\begin{figure*}[!t]
\centering

\includegraphics[width=.7\linewidth]{./figs/rq1/cross-opt/cross-opt-legend.pdf}

\subfloat[Function similarity]{
\includegraphics[width=0.32\linewidth]{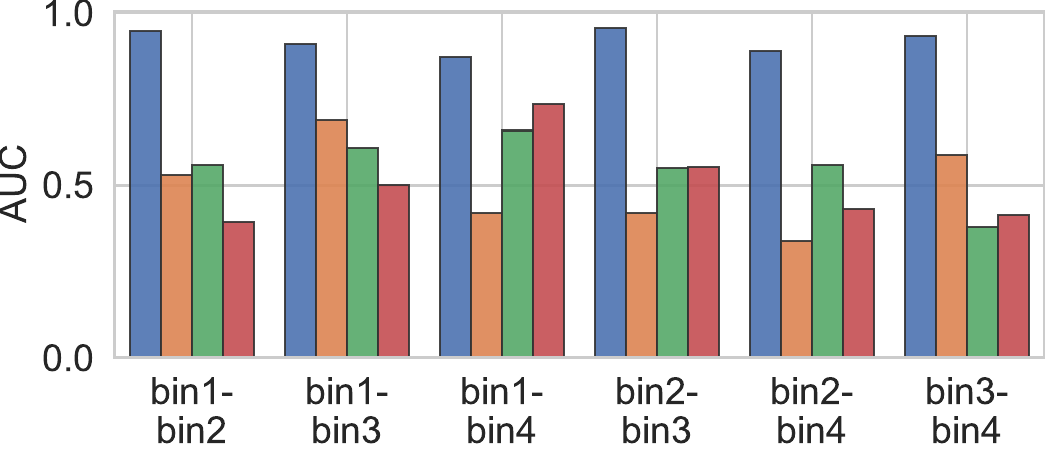}
\label{subfig:similarity-length}}\hfill
\subfloat[Function signature]{
\includegraphics[width=0.32\linewidth]{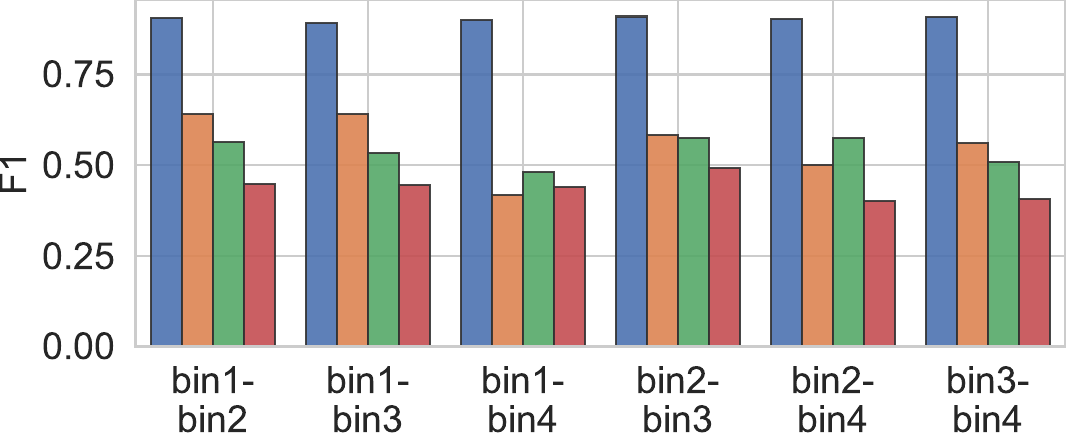}
\label{subfig:signature-length}}\hfill
\subfloat[Memory region]{
\includegraphics[width=0.32\linewidth]{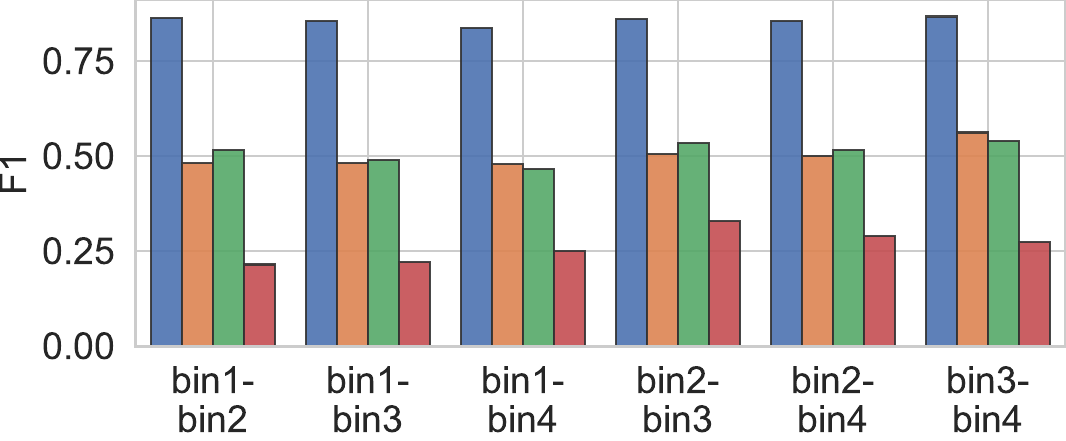}
\label{subfig:region-length}}

\caption{Evaluation on unseen samples with longer lengths. \texttt{bin1}-\texttt{bin4} denotes training on samples with lengths in \texttt{bin1} and testing on those in \texttt{bin4}.}
\label{fig:unseen-length}
\end{figure*}

\begin{table}[!t]

\centering
\setlength{\tabcolsep}{6pt}
\renewcommand{\arraystretch}{1}

\caption{The performance (F1) of \sys and baselines against the adversarial transformations transformations.}


\label{tab:adv}

\begin{tabular}{lrrr}
\toprule[1.1pt]
\textbf{} & \textbf{Orig.} & \textbf{Adv.} & \textbf{\begin{tabular}[c]{@{}l@{}}Invariance\\ Violation (\%)\end{tabular}} \\ \midrule[.9pt] 
\sys & \textbf{52.9} & \textbf{47.5} & \textbf{26} \\
GraphCodeBERT & 52.56 & 42.89 & 51\ \\
DOBF & 51.59 & 39.68 & 51 \\
CodeT5 & 44.21 & 36.66 & 47 \\ \bottomrule[1.1pt] 
\end{tabular}
\end{table}

\para{Adversarial robustness.}
In addition to randomly transforming samples, we consider adversarial attacks where the generation of semantics-preserving transformations is further guided by an objective that maximizes the changed predictions of the model.
In particular, we compare \sys and the baselines against the adversarial attack, \ie Averloc~\citep{ramakrishnan2020semantic}, for function name prediction. 
The adversarial transformations implemented in Averloc include a subset of the code transformations we considered, \eg variable renaming, dead code insertion, etc., with additional transformations such as loop unrolling. 
We follow the setting in Averloc by computing the adversarial attacks against a seq2seq model trained by the Averloc authors, and evaluate \sys and the baselines on the generated adversarial examples. 
This ensures a fair comparison by evaluating all the models on the same set of adversarial examples. 

Table~\ref{tab:adv} shows that \sys outperforms the second-best baseline, GraphCodeBERT, by 10.7\% and 49\%, in F1 and violation rate on the adversarial examples, respectively. 
This indicates the strong robustness of \sys against adversarial code transformations, even though the attacks are not statement permutations.

\subsection{Efficiency}
\label{subsec:efficiency}

\para{Overhead.}
To incorporate the inductive bias of code, many code models involve extracting and encoding code structures, including our baselines, \eg GraphCodeBERT and GGNN.
This is because computing PDGs statically is not overly expensive. 
Figure~\ref{fig:runtime-overhead} shows the runtime performance (in milliseconds) per code sample of \sys and PalmTree on extracting code structures, training, and inference, using the same exact hardware.

\begin{figure}[!t]
    \centering
    \includegraphics[width=.9\linewidth]{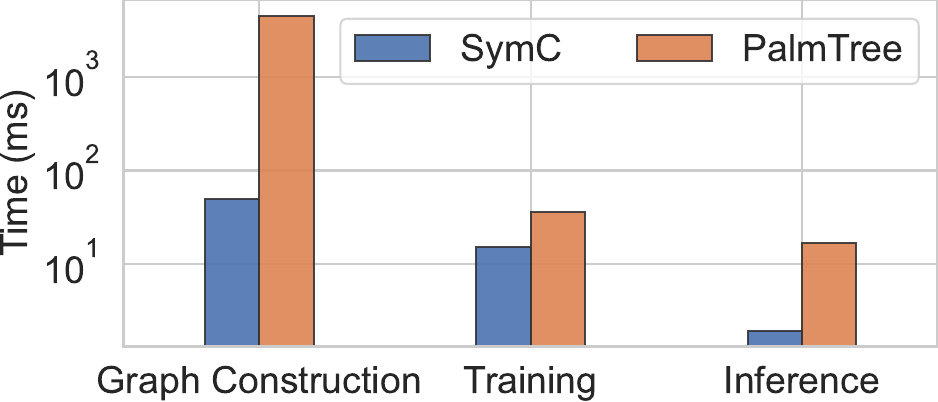}
    
    \caption{Comparing \sys to standard Transformer encoder in terms of the additional overhead introduced by constructing PDG, and training and inference with PDG-biased self-attention layers.}
    \label{fig:runtime-overhead}
\end{figure}

\sys's cheap computation of PDG incurs 88.8$\times$ less runtime overhead than PalmTree. 
However, our approach does incur additional computational cost for graph construction.
Therefore, it remains an interesting research problem to incorporate system optimization, \eg caching, to improve the efficiency of the PDG computation during inference.

\para{Training efficiency.}
We study the training effort (including both pre-training and fine-tuning) of \sys and \pto.
Table~\ref{tab:training-efficiency} shows their GPU hours, power, and emitted carbon dioxide estimation when they reach 0.5 F1 score in memory region prediction.
We assume the GPU always reaches its power cap (350W) to estimate an upper bound of the power usage.
CO$_2$eq stands for the carbon dioxide equivalent, a unit for measuring carbon footprints.
By being more training efficient, \sys incurs 1,281$\times$ less total GPU time, power, and emitted carbon dioxide than \pto in obtaining the same performance.

\begin{table}[!t]
    \setlength{\tabcolsep}{9pt}
    \renewcommand{\arraystretch}{1}
    
    \captionof{table}{The resource consumed by training \sys and other baselines to reach 0.5 F1 score in memory region prediction. }
    
    \label{tab:training-efficiency}
    
    \begin{center}
        \begin{tabular}{lrrr}
        \toprule[1.1pt]
        & \multirow{2}{*}{\begin{tabular}[r]{@{}r@{}}\textbf{Time}\\ (Hours)\end{tabular}} & \multirow{2}{*}{\begin{tabular}[r]{@{}r@{}}\textbf{Power}\\ (kWh)\end{tabular}} & \multirow{2}{*}{\begin{tabular}[r]{@{}r@{}}\textbf{Carbon}\\ (CO$_2$eq)\end{tabular}} \\ 
        &  &  & \\ \midrule[.9pt]
        \sys & \textbf{0.07} & \textbf{0.025} & \textbf{0.009} \\
        PalmTree-O$^*$ & 89.67 & 31.38 & 11.64 \\ \bottomrule[1.1pt]
        
        \multicolumn{4}{l}{\renewcommand{\arraystretch}{1}\begin{tabular}[c]{@{}l@{}}\scriptsize $^*$\pto did not disclose its hours for pre-training, so we include the\\\scriptsize pre-training time (in 10 epochs) based on our own pre-trained PalmTree. \end{tabular}}
        \end{tabular}
    \end{center}
\end{table}